\documentclass[10pt]{article}
\usepackage[utf8]{inputenc}
\usepackage[english]{babel}
\usepackage{hyperref}
\usepackage{textcomp}
\usepackage[autostyle]{csquotes}

\usepackage[backend=biber,sorting=none,giveninits=true, maxbibnames=99]{biblatex}
\AtBeginBibliography{\footnotesize}

\usepackage{csquotes}
\usepackage{amssymb}
\usepackage{amsthm}
\usepackage{amsfonts}
\usepackage{amssymb}
\usepackage{graphicx}
\usepackage{comment}
\usepackage{geometry}
\usepackage{subfiles}
\usepackage{subfig}
\usepackage{mathtools}
\usepackage{pgfplots}
\pgfplotsset{/pgf/number format/use comma,compat=newest}
\usepackage{url}
\usepackage{tikz}
\usepackage{notations}
\usepackage{bbm}
\usepackage{authblk}
\usepackage{floatflt}
\usepackage{subfig}

\usetikzlibrary{arrows,decorations.pathmorphing,backgrounds,positioning,fit,petri}

\usepgfplotslibrary{fillbetween}

\newcommand{\E}{\mathbb{E}}
\newcommand{\R}{\mathbb{R}}

\newcommand{\V}{\mathbb{V}}

\newcommand{\cZ}{\mathcal{Z}}

\theoremstyle{plain} 
\newtheorem{theorem}{Theorem} 

\newtheorem{corollary}[theorem]{Corollary} 
\newtheorem{lemma}[theorem]{Lemma} 
\newtheorem{proposition}[theorem]{Proposition} 

\theoremstyle{definition}

\theoremstyle{remark} 
\newtheorem{remark}{Remark}

\newcommand\iid{\mathrel{\stackrel{\makebox[0pt]{\mbox{\normalfont\tiny iid}}}{\sim}}}

\newcommand\distreq{\mathrel{\stackrel{\makebox[0pt]{\mbox{\normalfont\tiny D}}}{=}}}

\title{Fundamental limits of overparametrized  \\ shallow neural networks for supervised learning}
\author{Francesco Camilli, Daria Tieplova and Jean Barbier}
\affil{\small\emph{The Abdus Salam International Center for Theoretical Physics, Trieste, Italy}\\
}

\begin{document}
\maketitle

{
\let\thefootnote\relax\footnotetext{\url{ {fcamilli,dtieplov,jbarbier}@ictp.it}
}

\begin{abstract}
    We carry out an information-theoretical analysis of a two-layer neural network trained from input-output pairs generated by a teacher network with matching architecture, in overparametrized regimes. Our results come in the form of bounds relating $i)$ the mutual information between training data and network weights, or $ii)$ the Bayes-optimal generalization error, to the same quantities but for a simpler (generalized) linear model for which explicit expressions are rigorously known. Our bounds, which are expressed in terms of the number of training samples, input dimension and number of hidden units, thus yield fundamental performance limits for any neural network (and actually any learning procedure) trained from limited data generated according to our two-layer teacher neural network model. The proof relies on rigorous tools from spin glasses and is guided by ``Gaussian equivalence principles'' lying at the core of numerous recent analyses of neural networks. With respect to the existing literature, which is either non-rigorous or restricted to the case of the learning of the readout weights only, our results are information-theoretic (i.e. are not specific to any learning algorithm) and, importantly, cover a setting where \emph{all} the network parameters are trained. 
\end{abstract}

\section{Introduction}
Artificial neural networks (NNs) are universal approximators \cite{Cybenko1989ApproximationBS,HORNIK1989359,Universality2017} with remarkable abilities for supervised learning tasks such as regression or classification. In particular, modern deep neural networks, originally inspired by multilayer perceptrons  \cite{Gardner-derrida-perceptron,Sompolinski_learning_examples,NelderGLM,MCCULLAGH1984285,JeanGLM_PNAS}, achieve exceptional performance in image classification or speech recognition \cite{Alzubaidi2021ReviewOD} just to name a few examples. However, despite the important activity revolving around them, their theoretical understanding remains rather poor. 

One reason for the lack of strong theoretical guarantees for realistic NN models is related to the complex interplay between at least three aspects, whose individual effects are hard to single out: their architecture, the structure inherent to the data sets on which they are trained, as well as the algorithms and optimization procedures used to do so. It is therefore of crucial interest to tackle well defined theoretical models which are rich enough to capture some of the features of real NNs while remaining theoretically tractable. In this work, we propose to analyse a teacher-student set-up from a Bayesian-optimal perspective, with random input data and dependent responses generated according to a rule based on a teacher NN. This setting has the advantage to disentangle the aforementioned three components of NN learning by allowing us to mostly focus on how the \emph{architecture} of the NN used for learning (and data generation), and how the amount of accessible data, influence the prediction performance. More precisely, we are going to show that when learning a complex rule linking \emph{unstructured} inputs to responses in the \emph{information-theoretic optimal way}, and this in an \emph{overparametrized regime}, then an explicit characterization of the prediction capabilities of the NN is possible. Our results being of an information-theoretic nature, they will not depend on a specific learning procedure. Moreover because the inputs will be structure-less, the conclusions drawn will essentially capture architecture-dependent features of the learning; in the present case, the effect of overparametrization. 

A key challenge one has to face when analysing NNs is the presence of non-linear activation functions, whose role is essential for the network expressivity. Models with linear activations cannot capture non-linearities, but they serve as a starting point for deeper understanding \cite{SompolinskyPRX2022}. The case of a narrow hidden layer was already studied more than thirty years ago \cite{schwarze1992generalization,schwarze1993learning,monasson1995learning} and more recently in \cite{aubin2018committee}. However, in the more challenging regimes where \emph{all layers are large and of comparable sizes} it was observed (for instance in \cite{montanari_surprises,MontanariMei-RF-regression,goldt2020hiddenmanifold,reeves_MarcGET,Yue2022GEP}) that certain NNs models behave like finely tuned linear models regardless of the activation type, given sufficient regularity. From this observations, a whole set of \enquote{Gaussian equivalence principles} (GEPs) have emerged as valuable tools for handling non-linear activations in both rigorous and more heuristic approaches. GEPs leverage a well-known fact in high-dimensional probability: suitably rescaled low-dimensional projections of high-dimensional vectors with weakly correlated components exhibit Gaussian behavior. Classical results \cite{sudakov1978proj,diaconis1984asymptotics} and recent developments \cite{reeves2017conditionalCLT,meckes2012approximation} support the validity of GEPs in various high-dimensional inference contexts, such as in the description of certain observables for shallow neural networks \cite{goldt2020hiddenmanifold,reeves_MarcGET,Yue2022GEP}. 

However, the extent to which GEPs apply to the \emph{information-theoretic study of NNs where all weights are learned} remains uncertain. Certain scaling regimes relating the number of data samples and network weights must cause GEPs to break down, as
NNs do not always behave like linear models \cite{ghorbani2020neural}. This paper aims to bridge this gap by means of rigorous mathematical physics techniques developed in the study of spin glasses. We demonstrate the existence of a scaling regime for two-layer networks where GEPs are rigorously applicable, with the number of data playing a central role. As a result, we establish the information-theoretical equivalence between a two-layer NN and a generalized linear model, that hence share the same optimal generalization error.

\paragraph{Notations} Bold notations are reserved for vectors and matrices. By default a vector $\bx$ is a column vector, and its transpose $\bx^\intercal$ is therefore a row vector. Thus the usual $L_2$ norm $\|\bx\|^2=\bx^\intercal \bx$ and $\bx\bx^\intercal$ is a rank-one projector. $\EE_A$ is an expectation with respect to the random variable $A$; $\EE$ is an expectation with respect to all random variables entering the ensuing expression. For a function $F$ of one argument we denote $F'$ its derivative. Notations like $i\le N$ always implicitly assume that the index $i$ starts at $1$. We often compactly write $\EE(\cdots)^2=\EE[(\cdots)^2]\ge (\EE(\cdots))^2=\EE^2(\cdots)$ and similarly for other functions, we denote equivalently $\EE[f(\cdots)]$ and $\EE f(\cdots)$. For any functions $f$ and $g$, $f=O(g)$ means that there exists a constant $K$ such that $|f|\leq K|g|$; in other words we simply denote $O(g):=O(|g|)$ where in the r.h.s. we use the standard big O notation. Hence, taken for instance a Gaussian r.v. $Z\sim\mathcal{N}(0,1)$, $f=O(Z)$ does \emph{not} imply $\E f =0$, but rather $|\E f|\leq \E|f|\leq K\E|Z|$.

\paragraph{Acknowledgements} The authors were funded by the European Union (ERC, CHORAL, project number 101039794). Views and
opinions expressed are however those of the author only and do not necessarily reflect those of the European Union or the European
Research Council. Neither the European Union nor the granting authority can be held responsible for them. We wish to thank Marco Mondelli for suggesting meaningful references, and Rosalba Pacelli and Pietro Rotondo for fruitful clarifications on \cite{SompolinskyPRX2022,rotondo2022statmechDNN}.

\section{Setting and main results}
\subsection{Bayesian-optimal learning of a shallow network in the teacher-student set-up}
We consider supervised learning with a Bayesian two-layer neural network  in a teacher-student setup, with matched architecture for teacher (i.e., data-generating) and student (i.e., trained) models. To be precise, let the dataset be $\mathcal{D}_n= \{(\bX_\mu,Y_\mu)\}_{\mu=1}^{n}$, with inputs $\bX_\mu\in\mathbb{R}^{d}$ and responses $Y_\mu\in\mathbb{R}$. The $Y_\mu$ could also be taken from $\mathbb{R}^K$ with $K$ independent of $d,p,n$ without much changes. The teacher network generates the responses according to the following rule:
\begin{align}
    \label{eq:channel}
    Y_\mu=f\Big(\frac{\ba^{*\intercal}}{\sqrt{p}}\varphi\Big(\frac{\bW^*\bX_\mu}{\sqrt{d}}\Big); \bA_\mu\Big)+\sqrt{\Delta}Z_\mu\,.
\end{align}
Here, $\ba^*\in\R^{p}$ and $\bW^*\in\R^{p\times d}$. The readout function $f$ can include stochasticity, modeled through its second argument $\bA_\mu$ which are independent and identically distributed (i.i.d.) random variables in $\R^k$ with $k$ fixed, whose common distribution is denoted by $P_A$. Whenever it is not specified, real functions are applied component-wise to vectors, such as the non linearities $f,\varphi$. We assume the following regularity hypotheses:
\begin{description}
    \item[A1)\label{Assum:phi}] The activation function $\varphi:\R\mapsto \R,\,\varphi\in C^3$ is $c$-Lipschitz for some absolute constant $c$, is an odd function, and has bounded second and third derivatives.
    \item[A2)\label{Assum:f}] The readout function $f$ as well as its first and second derivatives are $P_A$-almost surely bounded.
\end{description}

% \begin{figure}[h!!]

% \centering
% \begin{tikzpicture}[scale=.8]

% \node (a) at (-5,-2) {};
% \node (b) at (-5,0) {};
% \node (c) at (-5,2){};
% \node (d) at (-5,4){};

% \node (e) at (0,-1){};
% \node (f) at (0,1){};
% \node (g) at (0,3){};

% \node (h) at (4,1){};

% \filldraw [red!60] (a) circle (4pt);
% \filldraw [red!60] (b) circle (4pt);
% \filldraw [red!60] (c) circle (4pt);
% \filldraw [red!60] (d) circle (4pt);

% \filldraw [red!60] (e) circle (4pt);
% \filldraw [red!60] (f) circle (4pt);
% \filldraw [red!60] (g) circle (4pt);

% \filldraw [red!60] (h) circle (4pt);

% \draw (a)--(e);
% \draw (a)--(f);
% \draw (a)--(g);

% \draw (b)-- (e);
% \draw (b)-- (f);
% \draw (b)-- (g);

% \draw (c)-- (e);
% \draw (c)-- (f);
% \draw (c)-- (g);

% \draw (d)-- (e);
% \draw (d)-- (f);
% \draw (d)-- (g);

% \draw (h)-- (e);
% \draw (h)-- (f);
% \draw (h)-- (g);

% \node at (-2,4) {$\bW^*$};
% \node at (-4,4.5) {$\bX_\mu$};
% \node at (1.5,2.8) {$\ba^*$};

% \draw[dashed] (e)--(f); 
% \draw[dashed] (g)--(f); 

% \node[label=$\varphi$] at (g){};
% \node[label=$f$] at (h){};
% \end{tikzpicture}
% \caption{Stylized representation of the model under study.}\label{fig:mainfig}
% \end{figure}

We draw the independent inputs $\bX_\mu\iid\mathcal{N}(0,I_{d})$ from the standard Gaussian measure. Furthermore, Gaussian label noise $Z_\mu\iid\mathcal{N}(0,1)$ is added in \eqref{eq:channel}, whose variance is tuned by  $\Delta>0$. Introducing the output kernel
\begin{align}
    \label{eq:Pout}
    P_{\rm out}(y\mid x):=\int P_A(d\bA)\frac{1}{\sqrt{2\pi\Delta}}\exp\Big(
    -\frac{1}{2\Delta}(y-f(x;\bA))^2
    \Big)
    \,,
\end{align}
once can see that the random outputs are generated independently, conditionally on the teacher parameters $\btheta^*=(\ba^*,\bW^*)$ and the inputs, as
\begin{align}
    Y_\mu\sim P_{\rm out}\Big(\cdot\mid \frac{\ba^{*\intercal}}{\sqrt{p}}\varphi\Big(\frac{\bW^*\bX_\mu}{\sqrt{d}}\Big)\Big)
    \,.
\end{align}
The probability distribution for the weights of the teacher network are drawn from a centered Gaussian factorized prior distribution. Taking them with different variances is possible but it does not add much to the result, so we consider them all equal to one for our purposes: $a_i^*, W_{ij}^*\iid\mathcal{N}(0,1)$. The same Gaussian law is used as prior distribution for the Bayesian student neural network model. In empirical risk minimization, a Gaussian prior would induce a  $L^2$ norm regularization for the weights. 

In this paper we will instead deal with Bayesian learning in the so-called \emph{Bayes-optimal setting}, which is the proper framework to quantify the fundamental performance limits of neural networks. Concretely, the Bayes-optimal scenario corresponds to a realizable matched setting where the student network has exactly the same architecture as the teacher network used to generate the responses in the data $\mathcal{D}_n$. 
The analysis in this setting therefore yields the best information-theoretically achievable student's generalization performance whenever optimally trained, namely, when using Bayesian learning based on the posterior distribution of the student's parameters. As a consequence,
\begin{quote}
\emph{our results lay down fundamental upper bounds for the performance of any neural networks, Bayesian or not, trained in any possible manner (including empirical risk minimization rather than Bayesian learning) or, more generally, for any learning procedure for the supervised task at hand. Moreover, both the readout and internal hidden layer are trained, with a number of hidden units that can grow proportionally to the input dimension.}
\end{quote}
To the best of our knowledge, this is the first rigorous result of this kind in this challenging scaling regime.

More formally, a student network is said to be Bayes-optimal if it employs the same output kernel $P_{\rm out}$ as used by the teacher network, or equivalently same number of layers and layers widths, readout $f$ and activation $\varphi$, label noise variance $\Delta$, as well as same prior law for its weights. Bayes-optimal learning is then based on the Bayes posterior of the network parameters $\btheta=(\ba,\bW)$ which reads 
\begin{align}\label{eq:posterior}
    dP(\btheta\mid\mathcal{D}_n)= \frac 1{ \cZ(\mathcal{D}_n)}
    \prod_{\mu=1}^{n}P_{\rm out}\Big(
    Y_\mu\mid \frac{\ba^{\intercal}}{\sqrt{p}}\varphi\Big(\frac{\bW \bX_\mu}{\sqrt{d}}\Big)
    \Big)D\btheta
\end{align}where for brevity
\begin{align}
    D\btheta:=\prod_{i=1}^{p}\frac{da_i}{\sqrt{2\pi}}e^{-\frac{a_i^2}{2}}\prod_{i=1}^{p}\prod_{j=1}^{d}\frac{dW_{ij}}{\sqrt{2\pi}}e^{-\frac{W_{ij}^2}{2}}=:D\ba D\bW\,.
\end{align}
We will often use the notation $D$ for densities of objects whose entries are i.i.d. standard Gaussian variables. The normalization of the posterior distribution will be referred to as \emph{partition function}:
\begin{align}\label{eq:partition_function}
    \cZ(\mathcal{D}_n):=\int D\btheta\exp\Big(\sum_{\mu=1}^{n}u_{Y_\mu}(s_\mu)\Big)
\end{align}where we have introduced the two further definitions $u_y(x)=\log P_{\rm out}(y\mid x)$ and
\begin{align}
    s_{\mu}:=\frac{\ba^{\intercal}}{\sqrt{p}}\varphi\Big(\frac{\bW \bX_\mu}{\sqrt{d}}\Big)\,,\quad S_\mu:=\frac{\ba^{*\intercal}}{\sqrt{p}}\varphi\Big(\frac{\bW^* \bX_\mu}{\sqrt{d}}\Big)\,.\label{7}
\end{align}
Note that the partition function is random through the randomness of the dataset $\mathcal{D}_n$. Then, (optimal) Bayesian learning means that the predictor $\hat Y_{\rm Bayes}(\bX_{\rm new})$ of the response associated with a fresh input test sample corresponds to the mean of the Bayes posterior distribution of the response given the training data:
\begin{align}
\hat Y_{\rm Bayes}(\bX_{\rm new}):=\E[Y_{\rm new}\mid \mathcal{D}_n,\bX_{\rm new}]=\int dY \,Y\, P_{\rm out}\Big(Y\mid \frac{\ba^{\intercal}}{\sqrt{p}}\varphi\Big(\frac{\bW \bX_{\rm new}}{\sqrt{d}}\Big)\Big) dP(\btheta\mid\mathcal{D}_n)\,.
\end{align}

We will sometimes employ the language of statistical mechanics. In particular we interpret the posterior distribution as a Boltzmann-Gibbs measure over degrees of freedom which are the network weights. We shall denote the expectations w.r.t. the posterior with the so-called Gibbs brackets $\langle\cdot\rangle$. For future convenience we introduce also its replicated version: for a function $g$ dependent of $k$ copies $(\btheta_b)_{b\le k}$ of the parameters,
\begin{align}
    \langle g\rangle^{\otimes k}:=\frac{1}{\cZ(\mathcal{D}_n)^k}\int \prod_{b=1}^k D\btheta_b \prod_{\mu=1}^{n}P_{\rm out}\Big(
    Y_\mu\mid \frac{\ba^{\intercal}_b}{\sqrt{p}}\varphi\Big(\frac{\bW_b \bX_\mu}{\sqrt{d}}\Big)
    \Big)\,g((\btheta_b)_{b\le k})\,,
\end{align}
that, with a slight abuse of notation, will still be denoted by $\langle\cdot\rangle$. From the above definition we see that the replicated Boltzmann-Gibbs measure is factorized for a given realization of the dataset, interpreted as quenched randomness in the analogy with spin glasses \cite{mezard1987spin}. Hence, \emph{replicas}, namely i.i.d. samples from the posterior measure, are independent conditionally on $\mathcal{D}_n$. However, when computing so-called quenched averages $\E\langle \cdot\rangle$ a further expectation $\E$ is taken w.r.t. the quenched data which couples the replicas. 

One of the main object of interest is the \emph{free entropy} (i.e., log-partition function) per sample, which is nothing else than minus the Shannon entropy $H(\mathcal{D}_n)$ of the data distribution per sample:
\begin{align}\label{eq:free_entropy}
    \bar{f}_n:=\frac{1}{n}\E\log \cZ(\mathcal{D}_n)=-\frac{1}{n}H(\mathcal{D}_n)\,,
\end{align}
where the expectation $\E$ is w.r.t. to the training data $\mathcal{D}_n=\{(\bX_\mu,Y_\mu)\}_{\mu=1}^{n}$. The normalization by $n$ is natural given that the number of terms in the ``energy'' defined by the exponent in \eqref{eq:partition_function} is precisely $n$. The data has a joint law that can be written in terms of the output kernel
\begin{align}
    \label{eq:joint_law_dataset}
    dP(\mathcal{D}_n)&=\prod_{\mu=1}^{n}\Big(\prod_{j=1}^{d}\frac{dX_{\mu j}}{\sqrt{2\pi}}e^{-\frac{X_{\mu j}^2}{2}} \Big)dY_\mu \,\E_{\ba^*,\bW^*}\prod_{\mu=1}^{n}P_{\rm out}(Y_\mu \mid S_\mu)\nonumber \\
    &=: \prod_{\mu=1}^{n} D \bX_\mu dY_\mu \, \E_{\ba^*,\bW^*}\exp\Big(\sum_{\mu=1}^{n}u_{Y_\mu}(S_\mu)\Big)\,.
\end{align}
Two observations are in order. First, the samples, indexed by $\mu$, are not independent because the responses were all drawn from the teacher, even though the $\bX_\mu$'s are independently generated. Second, except for the presence of differentials on the quenched variables this expression is very similar to the partition function \eqref{eq:partition_function}. This is due to the Bayes-optimality of the student network and has some pragmatic consequences, such as the Nishimori identities which will be key for the proof, see Appendix \ref{sec:nishiID}. 

Finally, note that for the sake of simplicity we did not include trainable biases in the definition of our NN model. However, we believe that adding them would not change much to our analysis as long as, like the other trainable parameters, they are Gaussian distributed and the student network is again Bayes-optimal and uses same architecture and number of parameters as the teacher model.

\subsection{An information-theoretically equivalent generalized linear model}
We now introduce an another model known as generalized linear model (GLM) \cite{NelderGLM,MCCULLAGH1984285,JeanGLM_PNAS}, which can be represented as a one layer neural network and which is thus a generalization of the classical perceptron \cite{Gardner-derrida-perceptron,Sompolinski_learning_examples}. One particular instance of the GLM turns out to be deeply connected to the setting with shallow networks introduced in the previous section. In this model the responses are generated independently conditionally on the inputs as
\begin{align}
    \label{eq:GLMchannel}
    Y_\mu^\circ=f\Big(\rho\frac{\bv^{*\intercal}\bX_\mu}{\sqrt{d}}+\sqrt{\epsilon}\xi_\mu^*\,;\,\bA_\mu \Big)+\sqrt{\Delta}Z_\mu\,,\quad\text{or}\quad Y_\mu^\circ\sim P_{\rm out}\Big(\cdot\mid \rho\frac{\bv^{*\intercal}\bX_\mu}{\sqrt{d}}+\sqrt{\epsilon}\xi_\mu^*\Big)
\end{align}
where $\bv^*=(v_j^*)_{j\leq d}\in\R^{d}$, $v_j^*\iid\mathcal{N}(0,1)$, $\xi_\mu^*\iid\mathcal{N}(0,1)$ all independently and the rest is as before. With respect to the two-layer neural network, the non-linearity brought by the middle layer has been replaced by a linear model with an additional effective Gaussian noise $\xi^*_\mu$. $\rho$ and $\epsilon\geq0$ are two real parameters that will be specified later.
This connection between the Bayes-optimal learning of neural networks and this GLM was recently conjectured in \cite{cui2023optimal} based on the replica method and the application of Gaussian equivalences. Our results vindicate their conjecture but for different scaling regimes relating the diverging parameters $d,p,n$. We are going to further comment on this point later on.

All the above construction can be repeated for the generalized linear model. From now on, quantities characterized by a $^\circ$ superscript will refer to the GLM. For starters, we denote the dataset generated through \eqref{eq:GLMchannel} by $\mathcal{D}_n^\circ:=\{(\bX_\mu,Y_\mu^\circ)\}_{\mu=1}^n$. 
%The posterior distribution for this model reads
% \begin{align}
%     dP(\bxi,\bv\mid\mathcal{D}_n^\circ)=\frac{1}{\mathcal{Z}^\circ(\mathcal{D}_n^\circ)}\prod_{\mu=1}^n P_{\rm out}\Big(Y_\mu^\circ
%     \mid s_\mu^\circ
%     \Big)D\bv \prod_{\mu=1}^nD\xi_\mu
% \end{align}where 
Let $$s_\mu^\circ=\rho\frac{\bv^{\intercal}\bX_\mu}{\sqrt{d}}+\sqrt{\epsilon}\xi_\mu\,, \qquad S_\mu^\circ=\rho\frac{\bv^{*\intercal}\bX_\mu}{\sqrt{d}}+\sqrt{\epsilon}\xi_\mu^*$$ and $D\bxi=\prod_\mu D\xi_\mu$. The expectation under the GLM posterior of any bounded test function $g$ of $k$ ``replicas'' (i.e., conditionally on the data i.i.d. copies) $(\bv_b,\bxi_b)_{b\le k}$ reads
\begin{align}
    \langle g\rangle^{\circ\,\otimes k}:=\frac{1}{\cZ^{\circ}(\mathcal{D}_n^\circ)^k}\int \prod_{b=1}^k D\bv_b D\bxi_b \prod_{\mu=1}^{n}P_{\rm out}\Big(
    Y_\mu^\circ\mid \rho\frac{\bv^{\intercal}_b\bX_\mu}{\sqrt{d}}+\sqrt{\epsilon}\xi_{b\mu }\Big)\,g((\bv_b,\bxi_b)_{b\le k})\,,
\end{align}
with $\cZ^{\circ}(\mathcal{D}_n^\circ)$ the GLM posterior normalization. As before, the free entropy reads
\begin{align}
    \bar{f}_n^\circ:=\frac{1}{n}\E\log\mathcal{Z}^\circ(\mathcal{D}_n^\circ)=\frac{1}{n}\E\log\int D\bv D\bxi\exp\Big(\sum_{\mu=1}^n u_{Y_\mu^\circ}(s_\mu^\circ)\Big)\,.
\end{align}
Finally, we write the distribution of the dataset, that is used for the quenched average in the above formula:
\begin{align}
    dP(\mathcal{D}_n^\circ)&=\prod_{\mu=1}^{n}\Big(\prod_{j=1}^{d}\frac{dX_{\mu j}}{\sqrt{2\pi}}e^{-\frac{X_{\mu j}^2}{2}} \Big)dY_\mu^\circ \,\E_{\bv^*,\bxi^*}\prod_{\mu=1}^{n}P_{\rm out}(Y_\mu^\circ \mid S^\circ_\mu)\nonumber \\
    &=: \prod_{\mu=1}^{n} D \bX_\mu dY_\mu \, \E_{\bv^*,\bxi^*}\exp\Big(\sum_{\mu=1}^{n}u_{Y^\circ_\mu}(S^\circ_\mu)\Big)
\end{align}
and the optimal Bayesian predictor is
\begin{align}
    \hat Y^\circ_{\rm Bayes}(\bX_{\rm new}):=\E[Y^\circ_{\rm new}\mid \mathcal{D}^\circ_n,\bX_{\rm new}]=\int dY \,Y\, P_{\rm out}\Big(Y\mid \rho\frac{\bv^{\intercal}\bX_{\rm new}}{\sqrt{d}}+\sqrt{\epsilon}\xi_{\mu}\Big) dP(\bv,\bxi\mid\mathcal{D}^\circ_n)\,.
\end{align}

\subsection{Results}
Our first theorem concerns the equivalence between the Bayes-optimal learning of the neural network and GLM models at the level of free entropy. Letting $Z\sim\mathcal{N}(0,1)$ we denote $\E_{\mathcal{N}(0,1)}g:=\E g(Z)$.
\begin{theorem}[Free entropy equivalence]
    \label{thm:equivalence} Let 
    \begin{align}
    \rho:=\E_{\mathcal{N}(0,1)} \varphi'  \quad \text{and} \quad \epsilon^2:=\E_{\mathcal{N}(0,1)}\varphi^2-\rho^2\,,    
    \end{align}
    and suppose \nameref{Assum:phi} and \nameref{Assum:f} hold. Then
    \begin{align}\label{boundFreeEntropy}
        |\bar{f}_n-\bar{f}^\circ_n|=O\Big(\sqrt{\Big(1+\frac{n}{d}\Big)\Big(\frac{n}{p}+\frac{n}{d^{3/2}}+\frac1{\sqrt d}\Big)}\Big)\,.
    \end{align}
\end{theorem}
From the previous statement we can identify the scaling regime for which the equivalence holds, namely, when the right hand side of \eqref{boundFreeEntropy} goes to $0$. From now on we shall denote by $$\widetilde\lim \, g_{d,p,n}:=\lim_{i\to \infty} \, g_{d_i,p_i,n_i}$$ where $(d_i,p_i,n_i)_i$ is any sequence of triplets of integers such that 
\begin{align}
\lim_{i\to \infty}\Big(1+\frac{n_i}{d_i}\Big)\Big(\frac{n_i}{p_i}+\frac{n_i}{d_i^{3/2}}+\frac1{\sqrt d_i}\Big)=0\,.  \label{thermoLimit}  
\end{align}

As a corollary we have the matching of the mutual informations between the dataset $\mathcal{D}_n$ and the network weights in the same limit. For the two-layer neural network the mutual information is related to the free entropy through the following expression ($H$ is the Shannon entropy)
\begin{align}\label{equiv1}
    \frac{1}{n}I_n(\btheta^*;\mathcal{D}_n)=\frac{1}{n}H(\mathcal{D}_n)-\frac{1}{n}H(\mathcal{D}_n\mid \btheta^*)=-\bar{f}_n+\E\log P_{\rm out}\Big(Y_1\mid\frac{\ba^{*\intercal}}{\sqrt{p}}\varphi\Big(\frac{\bW^*\bX_1}{\sqrt{d}}\Big)\Big)\, ,
\end{align}
whereas for the equivalent GLM we have
\begin{align}
    \begin{split}\label{equiv2}
        \frac{1}{n}I_n^\circ(\btheta^{\circ*};\mathcal{D}_n^{\circ})&=-\bar{f}^\circ_n+\E\log P_{\rm out}\Big(Y_1\mid\rho\frac{\bv^{*\intercal}\bX_1}{\sqrt{d}}+\sqrt{\epsilon}\xi_1^*\Big)\,.
    \end{split}
\end{align}
Considering that the teacher weights and inputs are Gaussian we have in law
\begin{align}
    \frac{\ba^{*\intercal}}{\sqrt{p}}\varphi\Big(\frac{\bW^*\bX_1}{\sqrt{d}}\Big)\distreq Z\sqrt{\frac{1}{p}\Big\Vert\varphi\Big(\frac{\bW^*\bX_1}{\sqrt{d}}\Big)\Big\Vert^2}
\end{align}
with $Z\sim\mathcal{N}(0,1)$ and $\| \cdot \|$ is the standard $L^2$ norm for vectors. Therefore, it is clear that the randomness in $\bW^*$ and $\bX_1$ will disappear in the limit. A similar equality holds for the GLM:
\begin{align}
    \rho\frac{\bv^{*\intercal}\bX_1}{\sqrt{d}}+\sqrt{\epsilon}\xi_1^*\distreq Z\sqrt{\rho^2\frac{\Vert\bX_1\Vert^2}{d}+\epsilon}\, .
\end{align}
Our goal is now to show that the arguments under square root both tend to $\rho^2+\epsilon=\E_{\mathcal{N}(0,1)}\varphi^2$ in the limit, and that we can plug this result inside the last terms in \eqref{equiv1} and \eqref{equiv2}, with a control on the error we make. To this end, define
\begin{align}
    S_d(t)=\sqrt{t\rho^2\Big(\frac{\Vert\bX_1\Vert^2}{d}-1\Big)+\epsilon+\rho^2}\,,\quad\text{or}\quad S_d(t)=\sqrt{t\Big(\frac{1}{p}\Big\Vert\varphi\Big(\frac{\bW^*\bX_1}{\sqrt{d}}\Big)\Big\Vert^2-\E_{\mathcal{N}(0,1)}\varphi^2\Big)+\E_{\mathcal{N}(0,1)}\varphi^2}
\end{align}
and
\begin{align}
    \Psi(t):=\E\int dY P_{\rm out}(Y\mid Z S_d(t))\log P_{\rm out}(Y\mid Z S_d(t))\,.
\end{align}
Using the properties later described in Lemma \ref{lem:propertiesPout} and the definition of $P_{\rm out}$ in \eqref{eq:Pout}, under assumptions (\nameref{Assum:phi}-(\nameref{Assum:f} one can readily verify that
\begin{align}
    |\dot{\Psi}(t)|\leq C(f)\E |Z| |\dot{S}_d(t)|\,,
\end{align}where $C(f)$ is a constant depending on $f$. From this bound, by the fundamental theorem of integral calculus, we have
\begin{align}
    |\Psi(1)-\Psi(0)|\leq C(f)\int_0^1 dt\,\E |\dot{S}_d(t)|\leq C(f)\begin{cases}
        \frac{\E\vert\Vert\varphi(\bW^*\bX_1/{\sqrt{d}})\Vert^2/p-\E_{\mathcal{N}(0,1)}\varphi^2\vert}{\sqrt{\E_{\mathcal{N}(0,1)}\varphi^2}}\quad &\text{for the 2-layer NN}\\
        \rho^2\frac{\E\vert\Vert\bX_1\Vert^2/d-1\vert}{\sqrt{\E_{\mathcal{N}(0,1)}\varphi^2}}\quad &\text{for the GLM}
    \end{cases}\,.
\end{align}
The remainder for the two-layer NN is $O(p^{-1/2}+d^{-1/4})$ whereas for the GLM we have $O(d^{-1/2})$. Finally, thanks to the previous argument
\begin{align}
    \frac{1}{n}I_n(\btheta^*;\mathcal{D}_n)&=-\bar{f}_n+\Psi(\E_{\mathcal{N}(0,1)}\varphi^2)+O(p^{-1/2}+d^{-1/4})\,,\\
    \frac{1}{n}I_n^\circ(\btheta^{\circ*};\mathcal{D}_n^{\circ})&=-\bar{f}_n^\circ+\Psi(\E_{\mathcal{N}(0,1)}\varphi^2)+O(d^{-1/2})\,,
\end{align}
where 
\begin{align}
    \Psi(\E_{\mathcal{N}(0,1)}\varphi^2):=\Psi(0)=\E\int dY P_{\rm out}(Y\mid Z\sqrt{\E_{\mathcal{N}(0,1)}\varphi^2})\log P_{\rm out}(Y\mid Z\sqrt{\E_{\mathcal{N}(0,1)}\varphi^2})\,.
\end{align}
Hence, we have just proved the information theoretical equivalence:
\begin{corollary}[Mutual information equivalence]
Under the same hypothesis of Theorem \ref{thm:equivalence} the following holds:
\begin{align}
    \Big|\frac{1}{n}I_n(\btheta^*;\mathcal{D}_n)-\frac{1}{n}I_n^\circ(\btheta^{\circ*};\mathcal{D}_n^{\circ})\Big|=O\Big(\sqrt{\Big(1+\frac{n}{d}\Big)\Big(\frac{n}{p}+\frac{n}{d^{3/2}}+\frac1{\sqrt{d}}\Big)}\Big)\,.
\end{align}
\end{corollary}

The performance of the neural network is quantified using the generalization error on test data using the square loss. The Bayes-optimal generalization error thus reads
\begin{align}
    \mathcal{E}_n:=\E\big(
    Y_{\rm new}-\E[Y_{\rm new}\mid \mathcal{D}_n,\bX_{\rm new}]
    \big)^2\,.
\end{align}
The outer expectation is intended w.r.t. the training data set and the test sample which is generated independently from the training data according to the same teacher model. The GLM Bayes-optimal generalization error $\mathcal{E}^\circ_n$ is defined similarly but considering the GLM teacher-student setting described in the previous section.

As a consequence of the proof technique of Theorem \ref{thm:equivalence} it is possible to show the following equivalence at the level of the generalization error, proven in Section \ref{sec:gen_error}.
\begin{theorem}[Generalization error equivalence]\label{prop:gen_error}
    Under the same hypothesis of Theorem \ref{thm:equivalence} the following holds:
    \begin{align}
        \widetilde\lim\,|\mathcal{E}_n-\mathcal{E}^\circ_n|=0\,,
    \end{align}
    i.e., the shallow neural network and noisy GLM settings lead to the same Bayes-optimal generalization error in any high-dimensional limit such that \eqref{thermoLimit} holds.
\end{theorem}

Our results combined with those of \cite{JeanGLM_PNAS} for the GLM provide explicit rigorous formulas for the mutual information and Bayes-optimal generalization error for the Bayesian neural network studied in this paper.

A remark is in order. The previous theorem states that the Bayes-optimal generalization error of a two-layer NN, which is trained on the dataset $\mathcal{D}_n$ generated by a teacher two-layer NN with same architecture, equals that of a noisy GLM trained on the dataset $\mathcal{D}_n^\circ$ generated by a matched teacher GLM. However, we cannot deduce from this that the two-layer NN trained using the GLM teacher dataset $\mathcal{D}_n^\circ$, or viceversa, achieves the Bayes-optimal generalization error. It would be interesting to investigate this aspect in the future.

\subsection{Related works}

\begin{figure}[t]
\centering
\begin{minipage}{0.11\textwidth}
    \begin{tikzpicture}[scale=.5]
    % Input Layer (Sixth Figure)
\foreach \x in {1,...,7}
    \node[circle,draw=black,fill=black] (input6\x) at (0,\x-4) {};

% Output Layer (Sixth Figure)
\node[circle,draw=black,fill=black] (output6) at (2,0) {};

% Connections (Sixth Figure)
\foreach \i in {1,...,7}
    \draw[red,very thick] (input6\i) -- (output6);

\end{tikzpicture}
\caption*{(1)}\label{GLM_fig}
\end{minipage}
\begin{minipage}{0.17\textwidth}
\begin{tikzpicture}[scale=.5]
    
% Input Layer (First Figure)
\foreach \x in {1,...,7}
\node[circle,draw=black,fill=black] (input\x) at (4,\x-4) {};

    % Hidden Layer (First Figure)
    \foreach \x in {1,2}
        \node[circle,draw=black,fill=black] (hidden\x) at (6,\x-1.5) {};
    
    % Output Layer (First Figure)
    \node[circle,draw=black,fill=black] (output) at (8,0) {};

    % Connections (First Figure)
    \foreach \i in {1,...,7}
        \foreach \j in {1,2}
            \draw[red,very thick] (input\i) -- (hidden\j);

    \foreach \j in {1,2}
        \draw[red, very thick] (hidden\j) -- (output);
   \end{tikzpicture}
   \caption*{(2)}\label{narrow}
   \end{minipage}
\begin{minipage}{0.17\textwidth}
\begin{tikzpicture}[scale=.5][scale=.5][scale=.5][scale=.5]
    % Input Layer (Second Figure)
    \foreach \x in {1,2}
        \node[circle,draw=black,fill=black] (input2\x) at (10,\x-1.5) {};
    
    % Hidden Layer (Second Figure)
    \foreach \x in {1,...,7}
        \node[circle,draw=black,fill=black] (hidden2\x) at (12,\x-4) {};
    
    % Output Layer (Second Figure)
    \node[circle,draw=black,fill=black] (output2) at (14,0) {};

    % Connections (Second Figure)
    \foreach \i in {1,2}
        \foreach \j in {1,...,7}
            \draw[red,very thick] (input2\i) -- (hidden2\j);

    \foreach \j in {1,...,7}
        \draw[red, very thick] (hidden2\j) -- (output2);
   
   \end{tikzpicture}
\caption*{(3)}\label{small_input}
   \end{minipage}
   \begin{minipage}{0.17\textwidth}
\begin{tikzpicture}[scale=.5][scale=.5][scale=.5]
    % Input Layer (Third Figure)
    \foreach \x in {1,2}
        \node[circle,draw=black,fill=black] (input3\x) at (16,\x-1.5) {};
    
    % Hidden Layer (Third Figure)
    \foreach \x in {1,...,7}
        \node[circle,draw=black,fill=black] (hidden3\x) at (18,\x-4) {};
    
    % Output Layer (Third Figure)
    \node[circle,draw=black,fill=black] (output3) at (20,0) {};

    % Connections (Third Figure)
    \foreach \i in {1,2}
        \foreach \j in {1,...,7}
            \draw[blue,very thick] (input3\i) -- (hidden3\j);

    \foreach \j in {1,...,7}
        \draw[red, very thick] (hidden3\j) -- (output3);

    \end{tikzpicture}
\caption*{(4)}\label{small_input_frozen}
   \end{minipage}
   \begin{minipage}{0.17\textwidth}
\begin{tikzpicture}[scale=.5][scale=.5]
% Input Layer (Fourth Figure)
\foreach \x in {1,...,7}
    \node[circle,draw=black,fill=black] (input4\x) at (22,\x-4) {};

% Hidden Layer (Fourth Figure)
\foreach \x in {1,...,7}
    \node[circle,draw=black,fill=black] (hidden4\x) at (24,\x-4) {};

% Output Layer (Fourth Figure)
\node[circle,draw=black,fill=black] (output4) at (26,0) {};

% Connections (Fourth Figure)
\foreach \i in {1,...,7}
    \foreach \j in {1,...,7}
        \draw[blue,very thick] (input4\i) -- (hidden4\j);

\foreach \j in {1,...,7}
    \draw[red, very thick] (hidden4\j) -- (output4); 
    \end{tikzpicture}
\caption*{(5)}\label{big_frozen}
   \end{minipage}
\begin{minipage}{0.17\textwidth}
\begin{tikzpicture}[scale=.5]
% Input Layer (Fifth Figure)
\foreach \x in {1,...,7}
    \node[circle,draw=black,fill=black] (input5\x) at (28,\x-4) {};

% Hidden Layer (Fifth Figure)
\foreach \x in {1,...,7}
    \node[circle,draw=black,fill=black] (hidden5\x) at (30,\x-4) {};

% Output Layer (Fifth Figure)
\node[circle,draw=black,fill=black] (output5) at (32,0) {};

% Connections (Fifth Figure)
\foreach \i in {1,...,7}
    \foreach \j in {1,...,7}
        \draw[red,very thick] (input5\i) -- (hidden5\j);

\foreach \j in {1,...,7}
    \draw[red, very thick] (hidden5\j) -- (output5);
    \end{tikzpicture}
\caption*{(6)}\label{big}
   \end{minipage}
\caption{Standard neural network architectures analyzed in the literature, classified according to how scale their layers widths (large diverging layers are depicted with many neurons, those which are much smaller or of fixed size have one or two neurons), and whether the internal weights are trainable (red edges) or completely fixed and not trained (blue edges). \textbf{(1)} represents a \emph{generalized linear model} (also called \emph{perceptron}), \textbf{(2)} corresponds to \textit{committee machines}, with large input size and small (finite) hidden layer; \textbf{(3)} and \textbf{(4)} are examples of \textit{mean field} regime, where the input size is small while the hidden layer is large; in particular model \textbf{(4)} corresponds to the \emph{random feature model};  \textbf{(5)} and \textbf{(6)} represent the challenging \textit{linear-width} regimes. Our results cover, e.g., models belonging to the regimes \textbf{(6)}, and \textbf{(3)} when $p\gg d \gg n \gg 1$.}
\label{fig:three_neural_networks}
\end{figure}

There exist by now a whole zoology of theoretical models for NNs studied in the literature and it becomes increasingly challenging to cover them whole. We provide here a partial classification of the main models divided according to how scale their internal widths compared to the inputs dimension, and whether the internal weights are trainable or not, see Figure \ref{fig:three_neural_networks}. For each class we provide a selection of relevant references without trying to be exhaustive.

\paragraph{Perceptrons and committee machines} The perceptron (and its generalization, i.e., the GLM) are linear classifiers with a non-linear readout. Committee machines can be viewed as two-layer neural networks with a narrow hidden layer and a single neuron output layer. These models have been studied in teacher-student set-ups and with online learning since the nineties \cite{barkai1992broken,Zippelius_committee_92,schwarze1993learning,schwarze1992generalization,schwarze1993generalization,monasson1995weight,Saad_Solla_committee_95,mato1992generalization,engel2001statistical,JeanGLM_PNAS,aubin2018committee,baldassi2019properties,goldt2020dynamics,goldt2020hiddenmanifold,reeves_MarcGET}. Despite their rich phenomenology with a so-called specialization phase transition where the model realizes the in-put-output rule is non-linear, these machines cannot capture the features of realistic data: they project high-dimensional data in a comparatively too low-dimensional space. The relevant regime for these models is $n,d\to\infty$ with $n/d\to\alpha\in(0,\infty)$ and $p=O(1)$. Despite the fact that all the weights among layers have to be learnt, w.r.t. our setting, the middle layer remains finite while the input dimension and number of data points diverge together. Note that, at first sight, it might be surprising that our result does not imply equivalence when $p$ is finite. However, for any $p>1$ such equivalences are not expected because a committee machine with two or more hidden units is not, in general, a linear classifier and can represent more complex non-linear relations. On the contrary, GEP-type of results are expected to provide reductions towards (generalized) linear models.

\paragraph{Mean-field regime} In the series of works \cite{nitanda2017stochastic,sirignano2020mean,araujo2019mean,nguyen2019mean,nguyen2020rigorous,mean-field-landscape-2-layer,mean-field-2-layer-kernal-limit,Chizat-Bach-2018,Rostkoff-Eijnded-23,Marco-Diyuan22,Marco-Montanari-AOS-MF,Marco-Alex-MF-piecewise,pmlr-shevchenko20a} the authors study the stochastic gradient descent (SGD) dynamics of multilayer neural networks. In contrast with the committee machine, here also the hidden layer can diverge in size (see \cite{mean-field-2-layer-kernal-limit}). This projection of a relatively low-dimensional signal into a very high-dimensional space has a regularizing property on the risk landscape. In particular, it causes the merging of possible multiple minima of the finite $p$ risk. SGD is then able to reach a near optimum with controllable guarantees. 

However, mean-field analyses of SGD dynamics differ from the information-theoretical one. Indeed, SGD produces a ``one-shot estimator'', which is in general outperformed by the Bayes-optimal one. Also, online learning is generally considered while our analysis consider (optimal) learning from a large fixed data set. Furthermore, for the information-theoretical equivalence in Theorem \eqref{thm:equivalence} to be valid, we need to control the size of the training set w.r.t. the network size, and in principle we can send both $d,p\to\infty$ with $d/p$ finite, as long as the training set is not too big ($n\ll p$).

\paragraph{Frozen hidden weights: neural tangent kernel, random features and lazy training} Neural tangent kernel (NTK) \cite{jacot-gabriel-hongler-2021} is a linearization of, say, a two-layer neural network, which reduces its training to a linear regression on the readout weights. As specified in \cite{mean-field-2-layer-kernal-limit}, NTK describes well the neural network performance at the initial stage of learning using SGD when the network weights are virtually frozen around their initialization. In a similar fashion, in random feature models and lazy training regimes \cite{RF-original,NIPS2017_a96b65a7,RF-rudi-rosasco,Bach-2017,li2018learning,allen2019convergence,du2018gradient,du2019gradient,lee2020wide,arora2019exact,huang2020dynamics,MontanariMei-RF-regression,ghorbani2019linearized, montanari2020interpolation,gerace2020GET_ICML,d2020double,geiger2020perspective,dhifallah2020precise,bordelon2020spectrum} the internal weights of the network are quenched, i.e., fixed once and for all. Other results \cite{neal2012bayesian,matthews2018gaussian,lee2018deep,novak2019bayesian,eldan2021non,bracale2021largewidth} show that large-width NNs with purely random weights behave as gaussian processes. Finally, recent works based on random matrix theory consider linear-width NNs but again under the assumption that only the last layer is learned, while internal ones are random 
\cite{NIPS2_GEP,Couillet_RMT_NNs,liao2018spectrum,benigni2021eigenvalue,peche2019note,fan2020spectra}.

Even though some of the above results, and more recently \cite{Yue2022GEP}, extend to extensive width input and hidden layers, as well as extensive number of data, i.e. $d,p,n\to\infty$ with $d,p,n$ all proportional, they hold in a setting that is fundamentally different from ours. In fact, we address the learning of \emph{all} the parameters in the network, considering all of them as \emph{annealed} variables, from a statistical mechanics perspective. Moreover, it is worth stressing again that we study the Bayes-optimal generalization error, and not the one coming from ERM. In this regard, it was shown in \cite{Yue_Lenka_ERM} that ERM with hinge or logistic losses can reach generalization errors that are close to Bayes-optimal in GLMs. In addition, as long as a suitable (though not convex) loss is taken into account, ERM can yield Bayes-optimal performances. However, this holds for GLMs, which have no hidden layer and thus only $p$ parameters to be learnt.

\paragraph{Linear-width regimes} A line of recent works \cite{SompolinskyPRX2022,rotondo2022statmechDNN,cui2023optimal} deals with the full training of the network as we do here. \cite{SompolinskyPRX2022} in particular carries out a thorough study for linear neural networks. In \cite{cui2023optimal}, instead, the authors conjecture the Bayes-optimal limits in the extensive width and data regime $d,p,n\to\infty$ all proportionally. Their computations are based on a combination of the heuristic replica method from spin glasses and a Gaussian equivalence principle, that allows to treat the non-linear activations in an efficient way. Despite GEPs have been proved rigorously in other contexts (for instance \cite{reeves_MarcGET,Yue2022GEP}), it is not obvious that they are directly applicable to the extensive width and data regime when the full training of the network is carried out. Indeed, it this not clear to us whether our proof can be extended to the whole regime considered in \cite{cui2023optimal}; in particular, we cannot assess whether the equivalence results provided in the next section do hold in the fully proportional regime where $d=\Theta(n)$ and large (this is allowed by our bounds) \emph{and} $p=\Theta(n)$ (which is instead prevented by our bounds) as considered \cite{cui2023optimal}, in spite that we cannot prove it at the moment.

Gaussian equivalence principles are also present in random matrix theory literature \cite{RMTGaussian,RMTGaussian2,FAN-GEP,NIPS2_GEP} and find applications in the study of random features models.

\paragraph{Estimation in multi-layer generalized linear models} Finally, we emphasize the difference between the learning problem considered in our work and the inference problem discussed in \cite{three-layer-GLM_Jean_Marylou}, later extended in \cite{multilayerGLM}, where the authors consider the task of reconstructing a vector from observations obtained from a multi-layer GLM with fixed, known, weight matrices. We believe that the proof of the concentration of the free entropy in \cite{multilayerGLM}, which was a bottleneck for GLM extensions to the multi-layer setting initiated in \cite{three-layer-GLM_Jean_Marylou}, can be adapted to our learning problem, yielding the Bayes-optimal generalization error GLM reduction for a deep network.

\section{Proof of Theorem \ref{thm:equivalence}}

\subsection{The interpolating model}
Our proof is based on the interpolation method, introduced in the seminal papers \cite{interp_guerra_2002,Guerra_upper_bound}. This method is a very effective tool whenever a comparison between two high dimensional models is needed. The idea is that of introducing an interpolating model, for any $t\in[0,1]$, that at its ends $t=0$ and $t=1$ matches the two models to be compared. In analogy with \cite{JeanGLM_PNAS}, we shall interpolate at the level of the variables $s_\mu$, $S_\mu$:
\begin{align}
    S_{t\mu}&:=\sqrt{1-t}\frac{\mathbf{a}^{*\intercal}}{\sqrt{p}}\varphi\Big(\frac{\mathbf{W}^*\mathbf{X}_\mu}{\sqrt{d}}\Big)+\sqrt{t}\rho\frac{\mathbf{v}^{*\intercal}\mathbf{X}_\mu}{\sqrt{d}}+\sqrt{t\epsilon}\xi_\mu^*\,,\\
    s_{t\mu}&:=\sqrt{1-t}\frac{\mathbf{a}^\intercal}{\sqrt{p}}\varphi\Big(\frac{\mathbf{W}\mathbf{X}_\mu}{\sqrt{d}}\Big)+\sqrt{t}\rho\frac{\mathbf{v}^{\intercal}\mathbf{X}_\mu}{\sqrt{d}}+\sqrt{t\epsilon}\xi_\mu\,.
\end{align}
We thus introduce an interpolating teacher and interpolating student, such that the second is Bayes-optimal for any $t\in[0,1]$. This allows us to use the so-called Nishimori identities of Appendix~\ref{sec:nishiID} uniformly in $t$.

The interpolating teacher network shall produce the $\mu=1,\ldots,n$ conditionally indpendent responses
\begin{align}
    Y_{t\mu}\sim P_{\rm out}(\cdot\mid
    S_{t\mu})\,,
\end{align}
where the output kernel is unchanged since it depends only on $f$. Posterior means now read
\begin{align}
    \langle g \rangle_t:=\dfrac{1}{\cZ_t}\int D\mathbf{a}D\mathbf{W}D\mathbf{v}D\bxi \exp\Big[\sum_{\mu=1}^{n}u_{Y_{t\mu}}(s_{t\mu})\Big]g
\end{align}
for any observable $g$ depending on $\ba,\bW,\bv,\bxi$, where
\begin{align}
    \cZ_t=\int D\mathbf{a}D\mathbf{W}D\mathbf{v}D\bxi\exp\Big[\sum_{\mu=1}^{n}u_{Y_\mu}(s_{t\mu})\Big]\,.
\end{align}
In the following we drop the subscript $t$ to keep the notation light and simply use $\langle \cdot\rangle$. We also introduce the compact notation
\begin{align}\label{def:gibbs_brack}
    \E_{(t)}(\cdot):=\E_{\ba^*}\E_{\setminus \ba^*}(\cdot)=\E_{\ba^*}
    \mathbb{E}_{\mathbf{W}^*,\mathbf{v}^*,\bxi^*,\{\mathbf{X}_\mu\}}
    \int \prod_{\mu=1}^{n}dY_{t\mu} 
    e^{u_{Y_{t\mu}}(S_{t\mu})}(\cdot)\,.
\end{align}

The free entropy of this interpolating model is thus
\begin{align}
    \bar{f}_n(t):=\frac{1}{n}\mathbb{E}_{(t)}\log\cZ_t\,,
\end{align}
whence it can be verified that
\begin{align}
    \bar{f}_n(0)=\bar{f}_n\,,\quad \bar{f}_n(1)=\bar{f}_n^\circ\,.
\end{align}
Due to the identity $\bar{f}_n(1)-\bar{f}_n(0)=\int_0^1 \frac{d}{dt}{\bar{f}}_n(t) dt$, a sufficient condition to prove our Theorem \ref{thm:equivalence} is to show that $\frac{d}{dt}{\bar{f}}_n(t)$ is uniformly bounded by the same order as in the statement. A direct computation shows 
\begin{align}\label{eq:der_free_entropy}
    \begin{split}
        \frac{d}{dt}{\bar{f}}_n(t)=-A_1+A_2+A_3+B
    \end{split}
\end{align}
where
\begin{align}
    &A_1:=\frac{1}{2n}\E_{(t)}\log \cZ_t
    \sum_{\mu=1}^{n}u^{\prime}_{Y_{t\mu}}(S_{t\mu})\frac{\mathbf{a}^{*\intercal}}{\sqrt{(1-t)p}}\varphi\Big(\frac{\mathbf{W}^*\mathbf{X}_\mu}{\sqrt{d}}\Big)\, ,\\
    &A_2:=\frac{1}{2n}\E_{(t)}\log \cZ_t
    \sum_{\mu=1}^{n}u^{\prime}_{Y_{t\mu}}(S_{t\mu})\rho\frac{\mathbf{v}^{*\intercal}\mathbf{X}_\mu}{\sqrt{td}}\, ,\\
    &A_3:=\frac{1}{2n}\E_{(t)}\log \cZ_t
    \sum_{\mu=1}^{n}u^{\prime}_{Y_{t\mu}}(S_{t\mu})\sqrt{\frac{\epsilon}{t}}\xi_\mu^*\, ,\\
    \label{eq:B-def}
     &B:=\frac{1}{n}\E_{(t)}\Big\langle \sum_{\mu=1}^{n}u^{\prime}_{Y_{t\mu}}(s_{t\mu})\frac {ds_{t\mu}}{dt} \Big\rangle\,.
\end{align}
We will control each term individually. For that we will need a number of Lemmas wich we provide now.

\subsection{Lemmata}
We collect here some Lemmas that shall be used intensively in the following. For the convenience we postpone proofs to the Appendix.

\begin{lemma}[Properties of $P_{\rm out}$]\label{lem:propertiesPout}
Recall the definition $u_y(x):=\log P_{\rm out}(y\mid x)$. We denote $u'_y(x):=\partial_x u_y(x)$. Furthermore, let 
\begin{align}
    \label{eq:U_munu}
    U_{\mu\nu}:=\delta_{\mu\nu} u''_{Y_{t\mu}}(S_{t\mu})+u'_{Y_{t\mu}}(S_{t\mu})u'_{Y_{t\nu}}(S_{t\nu})\,.
\end{align}
Under Assumptions (\nameref{Assum:phi} and (\nameref{Assum:f} the following statements hold:
\begin{align}
    &\E[u'_{Y_{t\mu}}(S_{t\mu})\mid S_{t\mu}]=\E[U_{\mu\nu}\mid S_{t\mu},S_{t\nu}]=0\,,\\
    &\E[(u'_{Y_{t\mu}}(S_{t\mu}))^2\mid S_{t\mu}]\,,\,\E[U_{\mu\nu}^2\mid S_{t\mu},S_{t\nu}]\leq C(f)\,,
\end{align}for a positive constant $C(f)$ depending solely on the readout function. 
\end{lemma}
\begin{remark}
    It is worth to point out a simple observation that for $\mu=\nu$ we have $U_{\mu\mu}=P^{\prime\prime}_{\rm out}(Y_{t\mu}\mid S_{t\mu})/P_{\rm out}(Y_{t\mu})$, where $$P^{\prime}_{\rm out}(y\mid x):= \partial_x P_{\rm out}(y\mid x)\,, \qquad P^{\prime\prime}_{\rm out}(y\mid x):= \partial_x\partial_x P_{\rm out}(y\mid x)\,, $$
    from what immediately follows 
    \begin{align}\label{ineq:P_out_xx}
        \E\Big[\Big(\frac{P^{\prime\prime}_{\rm out}(Y_{t\mu}\mid S_{t\mu})}{P_{\rm out}(Y_{t\mu}\mid S_{t\mu})}\Big)^2\mid S_{t\mu}\Big]\leq C(f)\,.
    \end{align}
\end{remark}

The following lemma will play a crucial role, and it contains all the approximations due to the law of large numbers. We introduce here a convenient notation for the pre-activations:
\begin{align}
    \boldsymbol{\alpha}_\mu:=\frac{\bW^*\bX_\mu}{\sqrt{d}}\,.
\end{align}
Hence, conditionally on the inputs $(\bX_\mu)_{\mu\leq n}$, the $\boldsymbol{\alpha}$'s have covariance
\begin{align}
    \frac{1}{p}\E[\boldsymbol{\alpha}_\mu^\intercal \boldsymbol{\alpha}_\nu\mid \bX_\mu,\bX_\nu]:= \frac{1}{p}\E_{\bW^*}\frac{(\bW^*\bX_\mu)^\intercal}{\sqrt{d}}\frac{\bW^*\bX_\nu}{\sqrt{d}}=\frac{\bX_\mu^\intercal\bX_\nu}{d}\,.
\end{align}

\begin{lemma}[Approximations] \label{APPROX_LEMMA}
Let $\tilde\varphi$ be either $\varphi$ or the identity function. Under assumptions (\nameref{Assum:phi} and (\nameref{Assum:f} the following estimates hold for any choice of $\mu,\nu\le n$:
    \begin{align}
    \label{approxphiprime}
        &\E_{\bW^*}\varphi'(\alpha_{\mu i})= \rho+O\Big(\frac{\Vert\bX_\mu\Vert^2}{d}-1\Big)\,,\\
        \label{approxphi_square}
        &\E_{\bW^*}\varphi^2(\alpha_{\mu i})=\E_{\mathcal{N}(0,1)}\varphi^2+O\Big(\frac{\Vert\bX_\mu\Vert^2}{d}-1\Big)\,,\\
        \label{approxphi_phi}
        &\E_{\bW^*}\varphi(\alpha_{\mu i}) \tilde\varphi(\alpha_{\nu i})=\rho\E_{\mathcal{N}(0,1)}\tilde\varphi'\frac{\bX_\mu^\intercal\bX_\nu}{d}+O\Big(\frac{\bX_\mu^\intercal\bX_\nu}{d}\Big(\frac{\Vert\bX_\mu\Vert^2}{d}-1\Big)\Big)+O\Big(\Big(\frac{\bX_\mu^\intercal\bX_\nu}{\Vert\bX_\nu\Vert^2}\Big)^2\Big)+O\Big(\frac{(\bX_\mu^\intercal\bX_\nu)^2}{\Vert\bX_\nu\Vert^2d}\Big)\,,\\
        \label{approxphiprime_phiprime}
        &\E_{\bW^*}\varphi'(\alpha_{\mu i})\varphi'(\alpha_{\nu i})= \rho^2+O\Big(\frac{\Vert\bX_\mu\Vert^2}{d}-1\Big)+O\Big(\frac{\bX_\mu^\intercal\bX_\nu }{\Vert\bX_\nu\Vert^2}\Big)\,,\\
        \label{approxphi_square_phi_square}
        &\E_{\bW^*}\varphi^2(\alpha_{\mu i})\tilde \varphi^2(\alpha_{\nu i})=\E_{\mathcal{N}(0,1)} \varphi^2\E_{\mathcal{N}(0,1)}\tilde\varphi^2
        +O\Big(\frac{\Vert\bX_\mu\Vert^2}{d}-1\Big)+O\Big(\frac{\bX_\mu^\intercal\bX_\nu }{\Vert\bX_\nu\Vert^2}\Big)\,.
    \end{align}
\end{lemma}

The final key ingredient is the concentration of the free entropy, that we prove in Section \ref{sec:concentration}.
\begin{theorem}[Free entropy concentration]\label{th:Z_concentr}
Under assumptions (\nameref{Assum:phi} and (\nameref{Assum:f} there exists a non-negative constant $C(f,\varphi)$ such that
\begin{align}
    \mathbb{E}_{ \mathbf{a}^*}\mathbb{V}_{\setminus\ba^*}\Big(\dfrac{1}{n}\log \cZ_t\Big)
    =\mathbb{E}\Big(\dfrac{1}{n}\log \cZ_t-\mathbb{E}_{\setminus\ba^*}\dfrac{1}{n}\log \cZ_t\Big)^2
    \leq C(f,\varphi)\Big(\dfrac{1}{d}+\dfrac{1}{n}\Big)\,.
\end{align}
\end{theorem}

\subsection{Proof of Theorem \ref{thm:equivalence}}

We split the proof of Theorem \ref{thm:equivalence} into different Lemmas for the sake of readability. If not differently specified, all the following Lemmas hold under the same hypotheses of Theorem \ref{thm:equivalence}. The first one concerns the $B$ contribution to the derivative of the free entropy \eqref{eq:der_free_entropy}.
\begin{lemma}[$B$ term]
 $B=0$.
\end{lemma}
\begin{proof}
The random variable inside the brackets in \eqref{eq:B-def} is a function of the data $Y_{t\mu}$ and of a sample from the posterior through $s_{t\mu}$. Hence we can use the Nishimori identities to get rid of the brackets, replacing $s_{t\mu}$ with the ground truth version $S_{t\mu}$ (from now on we denote with an upper dot $\dot S:=\frac{dS}{dt}$ the $t$-derivative):
\begin{align}
    B=\frac{1}{n}\E_{(t)}\sum_{\mu=1}^{n}u'_{Y_{t\mu}}(S_{t\mu})\dot{S}_{t\mu}=\frac{1}{n}\sum_{\mu=1}^{n} 
    \E_{(t)}\Big[\E_{(t)}[ u'_{Y_{t\mu}}(S_{t\mu})\mid S_{t\mu}]\dot{S}_{t\mu}\Big]
\end{align}
where we used the tower rule for expectations. The latter is identically zero thanks to Lemma \ref{lem:propertiesPout}.
\end{proof}

We split $A_1$ into two other contributions $A_1=A_{11}+A_{12}$ where
 \begin{align}
     &A_{11}:=\frac{1}{2n\sqrt{1-t}}\E_{(t)}\log \cZ_t
    \sum_{\mu=1}^{n}u^{\prime}_{Y_{t\mu}}(S_{t\mu})\Big(\frac{\mathbf{a}^{*\intercal}}{\sqrt{p}}\varphi\Big(\frac{\mathbf{W}^*\mathbf{X}_\mu}{\sqrt{d}}\Big)-\frac{{\rho}\mathbf{a}^{*\intercal} \mathbf{W}^* \mathbf{X}_\mu}{\sqrt{pd}}\Big)\,,\\
    &A_{12}:=\frac{1}{2n\sqrt{1-t}}\E_{(t)}\log \cZ_t
    \sum_{\mu=1}^{n}u^{\prime}_{Y_{t\mu}}(S_{t\mu})\frac{ {\rho}\mathbf{a}^{*\intercal} \mathbf{W}^* \mathbf{X}_\mu}{\sqrt{pd}}\,.
 \end{align}
Let us simplify these terms by Gaussian integration by parts. In $A_{12}$, integrating by parts w.r.t. $\mathbf{W}^*$ yields
\begin{align}
    A_{12}&=\frac{{\rho}}{2n}\E_{(t)}\log \cZ_t\sum_{\mu,\nu=1}^{n}U_{\mu\nu}\frac{\mathbf{a}^{*\intercal}\big(\mathbf{a}^*\circ\varphi'\big(\frac{\mathbf{W}^*\mathbf{X}_\nu}{\sqrt{d}}\big)\big)}{p}\frac{\mathbf{X}_\mu^{\intercal}\mathbf{X}_\nu}{d} \label{A12}
\end{align}
with $U_{\mu\nu}$ defined in Lemma \ref{lem:propertiesPout} and $\circ$ denotes the entry-wise (Hadamard) product. Concerning $A_{11}$, because of the non-linearity, we can only integrate by parts w.r.t. $\mathbf{a}^*$ and obtain
\begin{align}
\begin{split}
    A_{11}&=\frac{1}{2n}\E_{(t)}\log \cZ_t\sum_{\mu,\nu=1}^{n}U_{\mu\nu}\Big[ \frac{\varphi(\boldsymbol{\alpha}_\mu)^{\intercal}\varphi(\boldsymbol{\alpha}_\nu)-\rho\boldsymbol{\alpha}_\mu^\intercal\varphi(\boldsymbol{\alpha}_\nu)}{p}\Big]\,.
\end{split}
\end{align}
The off-diagonal $\mu\neq\nu$ and diagonal terms in the previous equations play two very different roles, and they shall be treated separately in the following.

\begin{lemma}[Off-diagonal part of $A_{11}$]\label{lem:offdiag1}
The following holds: 
    \begin{equation}\label{eq:off_diag_terms}
        \begin{split}
            A_{11}^{\rm off}:=\frac{1}{n}\E_{(t)}\log \cZ_t\sum_{\mu\neq\nu}U_{\mu\nu}\Big[\frac{\varphi(\boldsymbol{\alpha}_\mu)^\intercal\varphi(\boldsymbol{\alpha}_\nu)-\rho\boldsymbol{\alpha}_\mu^\intercal\varphi(\boldsymbol{\alpha}_\nu)}{p}\Big]=O\Big(\sqrt{\Big(1+\frac{n}{d}\Big)\Big(\frac{n}{p}+\frac{n}{d^{3/2}}\Big)}\Big)\,.
        \end{split}
    \end{equation}
\end{lemma}

\begin{proof}
    Let us start noticing that, for any smooth function $F(\boldsymbol{\alpha}_\mu,\boldsymbol{\alpha}_\nu)$ we have  
    \begin{align}
    \begin{split}
        \E_{\setminus\ba^*}U_{\mu\nu}F(\boldsymbol{\alpha}_\mu,\boldsymbol{\alpha}_\nu)=
        \E_{\setminus\ba^*}\big[\E_{\setminus\ba^*}[U_{\mu\nu}\mid \mathbf{W}^*,\bv^*,\boldsymbol{\xi}^*,\mathbf{X}]F(\boldsymbol{\alpha}_\mu,\boldsymbol{\alpha}_\nu)\big]=0
    \end{split}
    \end{align}
    thanks to Lemma \ref{lem:propertiesPout}. As a consequence, with $\ba^*$ fixed, we can modify $A_{11}^{\rm off}$ as follows without changing its value:
    \begin{align}
        A_{11}^{\rm off}=\E_{\setminus\ba^*}(f_n-\E_{\setminus\ba^*}f_n)\sum_{\mu\neq\nu}U_{\mu\nu}\Big[\frac{\varphi(\boldsymbol{\alpha}_\mu)^\intercal\varphi(\boldsymbol{\alpha}_\nu)-{\rho}\boldsymbol{\alpha}_\mu^\intercal\varphi(\boldsymbol{\alpha}_\nu)}{p}\Big]
    \end{align}
    with $f_n:=\log \cZ_t/{n}$. In the following we shall simply write $u'_\mu$ in place of $u'_{Y_{t\mu}}(S_{t\mu})$, and $\boldsymbol{\varphi}_\mu$ instead of $\varphi(\boldsymbol{\alpha}_\mu)$ for brevity. We are now in position to use Cauchy-Schwartz's identity:
    \begin{align}
        (A_{11}^{\rm off})^2\leq \mathbb{V}_{\setminus\ba^*}[f_n] \sum_{\mu\neq\nu}\sum_{\lambda\neq \eta} \E_{\setminus\ba^*} U_{\mu \nu} U_{\lambda \eta} \Big[\frac{\boldsymbol{\varphi}_\mu^\intercal\boldsymbol{\varphi}_\nu-{\rho}\boldsymbol{\alpha}_\mu^\intercal\boldsymbol{\varphi}_\nu}{p}\Big]\Big[\frac{\boldsymbol{\varphi}_\lambda^\intercal\boldsymbol{\varphi}_\eta-{\rho}\boldsymbol{\alpha}_\lambda^\intercal\boldsymbol{\varphi}_\eta}{p}\Big]\,.
    \end{align}
    Note that when all the four Greek indices are different from one another we get the highest combinatorial factor of $O(n^4)$. However, using the conditional independence of the responses and Lemma \ref{lem:propertiesPout}, the expectation sets them to $0$. Hence, the only contributions from the double sums come from $\mu=\lambda$ and $\nu=\eta$, or $\mu=\eta$ and $\nu=\lambda$, which gives twice the same quantity. Thus
    \begin{align}
        (A_{11}^{\rm off})^2\leq \mathbb{V}_{\setminus\ba^*}[f_n] \frac{2}{p^2}\sum_{\mu\neq\nu} \E_{\setminus\ba^*} (u'_\mu u'_\nu)^2 \sum_{i,j=1}^{p}\big[\varphi_{\mu i}\varphi_{\nu i}\varphi_{\mu j}\varphi_{\nu j}
        -2{\rho}\alpha_{\mu i}\varphi_{\nu i}\varphi_{\mu j}\varphi_{\nu j}+\rho^2\alpha_{\mu i}\varphi_{\nu i}\alpha_{\mu j}\varphi_{\nu j}\big]\,.\label{61}
    \end{align}
    The double sum on $i,j$ comes from the square of a scalar product. Lemma \ref{lem:propertiesPout} then allows to bound $\E[U_{\mu\nu}^2\mid S_{t\mu},S_{t\nu}]$ by a constant.
    
    Let us treat the off-diagonal terms ($i\neq j$) first. We call Lemma \ref{APPROX_LEMMA}, in particular \eqref{approxphi_phi}, to simplify the first term in \eqref{61}:
    \begin{align}
    \begin{split}
        &\E_{\setminus\ba^*}\sum_{i\neq j,1}^{p}\varphi({\alpha}_{\mu i})\varphi({\alpha}_{\nu i})\varphi({\alpha}_{\mu j})\varphi({\alpha}_{\nu j})= p(p-1)\E_{\bX_\mu,\bX_\nu}\big(\E_{\bW^*}[\varphi({\alpha}_{\mu 1})\varphi({\alpha}_{\nu 1})]\big)^2\\
        &= p(p-1) \E\Big[\rho^4 \Big(\frac{\bX_\mu^\intercal\bX_\nu}{d}\Big)^2+O\Big(\frac{\bX_\mu^\intercal\bX_\nu}{d}\Big(\frac{\bX_\mu^\intercal\bX_\nu}{\Vert \bX_\nu\Vert^2}\Big)^2\Big)+O\Big(\Big(\frac{\bX_\mu^\intercal\bX_\nu}{d}\Big)^2\frac{\bX_\mu^\intercal\bX_\nu}{\Vert \bX_\nu\Vert^2}\Big)+O\Big(\Big(\frac{\bX_\mu^\intercal\bX_\nu}{d}\Big)^2\Big(\frac{\Vert\bX_\nu\Vert^2}{d}-1\Big)\Big)\Big]\,.
    \end{split}
    \end{align}
    The first term in the square brackets corresponds to the square of the leading term in \eqref{approxphi_phi} with $\tilde\varphi=\varphi$. The other two terms are obtained as cross products between the leading term in \eqref{approxphi_phi} and the remainders.
    
    Exploiting the fact that the norm of a Gaussian vector concentrates with exponential speed, i.e.,
    \begin{align}
        \mathbb{P}\Big(\Big|\frac{\Vert\bX_\mu\Vert^2}{d}-1\Big|\geq h\Big)\leq\exp\Big(-\frac{dL h^2}{2}\Big)\,,\quad L>0\,,\forall h>0\,,
    \end{align}
    one can conclude that
    \begin{align}
        \E_{\setminus\ba^*}\sum_{i\neq j,1}^{p}\varphi({\alpha}_{\mu i}) \varphi({\alpha}_{\nu i})\varphi({\alpha}_{\mu j}) \varphi({\alpha}_{\nu j})=p(p-1)\Big[\frac{\rho^4}{d}+O\Big(\frac{1}{d^{3/2}}\Big)\Big]\,.
    \end{align}

    We now turn to the second term of \eqref{61}: using again a Gaussian integration by part for the second equality below followed by Lemma \ref{APPROX_LEMMA} we get
    \begin{align}
    \begin{split}
        {\rho}\E_{\bX_\mu,\bX_\nu}\E_{\bW^*}&\sum_{i\neq j}^{p}\alpha_{\mu i}\varphi_{\nu i}\varphi_{\mu j}\varphi_{\nu j}={\rho}p(p-1)\E_{\bX_\mu,\bX_\nu}\E_{\bW^*}[\alpha_{\mu 1}\varphi_{\nu 1}]\E_{\bW^*}[\varphi_{\mu 1}\varphi_{\nu 1}]\\
        &= {\rho}p(p-1)\E_{\bX_\mu,\bX_\nu}\frac{\bX_\mu^\intercal\bX_\nu}{d}\E_{\bW^*}[\varphi'_{\nu 1}]\E_{\bW^*}[\varphi_{\mu 1}\varphi_{\nu 1}]\\
        &= p(p-1)\E_{\bX_\mu,\bX_\nu} \Big[\rho^2\frac{\bX_\mu^\intercal\bX_\nu}{d}+O\Big(\frac{\bX_\mu^\intercal\bX_\nu}{d}\Big(\frac{\Vert\bX_\mu\Vert^2}{d}-1\Big)\Big)\Big]\\
        &\qquad\times \Big[\rho^2\frac{\bX_\mu^\intercal\bX_\nu}{d}+O\Big(\Big(\frac{\bX_\mu^\intercal\bX_\nu}{\Vert\bX_\nu\Vert^2}\Big)^2\Big)+O\Big(\frac{(\bX_\mu^\intercal\bX_\nu)^2}{\Vert\bX_\nu\Vert^2d}\Big)+O\Big(\frac{\bX_\mu^\intercal\bX_\nu}{d}\Big(\frac{\Vert\bX_\mu\Vert^2}{d}-1\Big)\Big)\Big]
    \end{split}
    \end{align}
    which shows that
    \begin{align}
        \rho\E_{\bX_\mu,\bX_\nu}\E_{\bW^*}&\sum_{i\neq j}^{p}\alpha_{\mu i}\varphi_{\nu i}\varphi_{\mu j}\varphi_{\nu j}=p(p-1)
        \Big[\frac{\rho^4 }{d}+O\Big(\frac{1}{d^{3/2}}\Big)\Big]\,.
    \end{align}
    Finally, for what concerns the off-diagonal terms $i\neq j$, we deal with the last term of \eqref{61}:
    \begin{align}
        \rho^2\sum_{i\neq j}^{p}\E_{\bX_\mu,\bX_\nu}\alpha_{\mu i}\varphi_{\nu i}\alpha_{\mu j}\varphi_{\nu j}&=\rho^2p(p-1)\E_{\bX_\mu,\bX_\nu}\E_{\bW^*}[\alpha_{\mu 1}\varphi_{\nu 1}]\E_{\bW^*}[\alpha_{\mu 1}\varphi_{\nu 1}]\nonumber \\
        &=p(p-1)\rho^2 \E_{\bX_\mu,\bX_\nu}\Big(\frac{\bX_\mu^\intercal\bX_\nu }{d}\Big)^2\E_{\bW^*}\varphi'_{\mu 1}\E_{\bW^*}\varphi'_{\nu 1}\nonumber\\
        &=p(p-1)\Big[\frac{\rho^4}{d}+O\Big(\frac{1}{d^{3/2}}\Big)\Big]
    \end{align}
    where we used integration by parts and the approximation Lemma \ref{APPROX_LEMMA}. From this computation we see that, remarkably, the leading orders of the off-diagonal terms $i\neq j$ in \eqref{61} cancel each other, leaving the more convenient rate $O(1/d^{3/2})$. More precisely, there exists an absolute constant $K$ such that
    \begin{align}
        (A_{11}^{\rm off})^2\leq \mathbb{V}_{\setminus\ba^*}[f_n] \frac{2K}{p^2}\sum_{\mu\neq\nu} \E_{\setminus\ba^*} \Big\{\sum_{i=1}^{p}\big[\varphi^2_{\mu i}\varphi^2_{\nu i}
        -2{\rho}\alpha_{\mu i}\varphi^2_{\nu i}\varphi_{\mu i}+\rho^2\alpha_{\mu i}^2\varphi_{\nu i}^2\big]+O\Big(\frac{p^2}{d^{3/2}}\Big)\Big\}\,.
    \end{align}
    From the previous bound we see that we cannot hope that the same cancellation occurs in the diagonal terms $i=j$. Using again the results from Lemma \ref{APPROX_LEMMA} one can show that
    \begin{align}
        (A_{11}^{\rm off})^2\leq \mathbb{V}_{\setminus\ba^*}[f_n] \frac{2K}{p^2}\sum_{\mu\neq\nu}\Big[O(p)+O\Big(\frac{p^2}{d^{3/2}}\Big)\Big]\,.
    \end{align}
    The statement is thus proved after we take care of the remaining expectation over $\ba^*$ using Theorem \ref{th:Z_concentr}:
    \begin{align}
        |\E_{\ba^*}A_{11}^{\rm off}|\leq\E_{\ba^*}\sqrt{{(A_{11}^{\rm off})^2}}\leq \sqrt{\E_{\ba^*}\mathbb{V}_{\setminus\ba^*}[f_n]O\Big(
        \frac{n^2}{p}+\frac{n^2}{d^{3/2}}
        \Big)}
        =O\Big(\sqrt{\frac{n}{p}+\frac{n}{d^{3/2}}+\frac{n^2}{dp}+\frac{n^2}{d^{5/2}}}\Big)\,.
    \end{align}
\end{proof}
\begin{remark}
The previous result is telling us that the number of data points $n$ can grow as fast as $o(d^{3/2})$ with the size of the input layer, but has to be much smaller than $p$, the size of the hidden layer. Furthermore, treating the difference $\varphi(\boldsymbol{\alpha}_\mu)^\intercal\varphi(\boldsymbol{\alpha}_\nu)- {\rho}\boldsymbol{\alpha}_\mu^\intercal\varphi(\boldsymbol{\alpha}_\nu)$ altogether is fundamental to obtain the scaling $n^2{d^{-3/2}}$ instead of $n^2d^{-1}$. We also stress that the pre-activations $\boldsymbol{\alpha}_\mu$ in the hidden layer have correlations among them that scale as $d^{-1/2}$. If $d$ is not big enough they cannot be considered weakly correlated.
\end{remark}

From Lemma \ref{lem:offdiag1} we thus infer that
\begin{align}
    A_{11}&=\frac{1}{2n}\E_{(t)}\log \cZ_t\sum_{\mu=1}^{n}\frac{P_{\rm out}''(Y_{t\mu} \mid S_{t\mu})}{P_{\rm out}(Y_{t\mu} \mid S_{t\mu})}\Big[ \frac{\Vert\varphi(\boldsymbol{\alpha}_\mu)\Vert^2-{\rho}\boldsymbol{\alpha}_\mu^\intercal\varphi(\boldsymbol{\alpha}_\mu)}{p}\Big]+O\Big(\sqrt{\Big(1+\frac{n}{d}\Big)\Big(\frac{n}{p}+\frac{n}{d^{3/2}}\Big)}\Big)\,.
\end{align}
For the term $A_3$ we use integration by parts  with respect to the variables Gaussian $\xi_{\mu}^*$:
\begin{align}
    A_3=\frac{\epsilon}{2n}\E_{(t)}\log \cZ_t
    \sum_{\mu=1}^{n}\Big((u^{\prime}_{Y_{t\mu}}(S_{t\mu}))^2+u^{\prime\prime}_{Y_{t\mu}}(S_{t\mu})\Big)
    =\frac{\epsilon}{2d}\E_{(t)}\log \cZ_t\sum_{\mu=1}^{n}\dfrac{P^{\prime\prime}_{\rm out}(Y_{t\mu}\mid S_{t\mu})}{P_{\rm out}(Y_{t\mu}\mid S_{t\mu})}\,.
\end{align}
Hence
\begin{multline}
    A_3-A_{11}=\frac{1}{2n}\E_{(t)}\log \cZ_t\sum_{\mu=1}^{n}\dfrac{P^{\prime\prime}_{\rm out}(Y_{t\mu}\mid S_{t\mu})}{P_{\rm out}(Y_{t\mu}\mid S_{t\mu})}\Big[\epsilon-\frac{\Vert\varphi(\boldsymbol{\alpha}_\mu)\Vert^2-{\rho}\boldsymbol{\alpha}_\mu^\intercal\varphi(\boldsymbol{\alpha}_\mu)}{p}\Big]
+O\Big(\sqrt{\Big(1+\frac{n}{d}\Big)\Big(\frac{n}{p}+\frac{n}{d^{3/2}}\Big)}\Big)\,.
\end{multline}

\begin{lemma}[$A_3-A_{11}^{\rm diag}$ term]\label{lem:epsilon_cancellation}
The following asymptotics holds:
\begin{align}
    A_3-A_{11}^{\rm diag}:=\frac{1}{2n}\E_{(t)}\log \cZ_t\sum_{\mu=1}^{n}\dfrac{P^{\prime\prime}_{\rm out}(Y_{t\mu}\mid S_{t\mu})}{P_{\rm out}(Y_{t\mu}\mid S_{t\mu})}\Big[\epsilon-\frac{\Vert\varphi(\boldsymbol{\alpha}_\mu)\Vert^2-{\rho}\boldsymbol{\alpha}_\mu^{\intercal}\varphi(\boldsymbol{\alpha}_\mu)}{p}\Big]
    =O\Big(\sqrt{\Big(\frac n d + 1\Big)\Big(\frac1p + \frac1{\sqrt d}\Big)}\Big)\,.
\end{align}
\end{lemma}

\begin{proof}
    Define
    \begin{align}
        C:=\frac{1}{n}\E_{\setminus\ba^*}\log \cZ_t\sum_{\mu=1}^{n}\frac{P^{\prime\prime}_{\rm out}(Y_{t\mu}\mid S_{t\mu})}{P_{\rm out}(Y_{t\mu}\mid S_{t\mu})}\Big[\epsilon-\frac{\Vert\varphi(\boldsymbol{\alpha}_\mu)\Vert^2-\rho\boldsymbol{\alpha}_\mu^{\intercal}\varphi(\boldsymbol{\alpha}_\mu)}{p}\Big]\,.
    \end{align}
    First, thanks to Lemma \ref{lem:propertiesPout}
    \begin{align}
        \begin{split}\label{eq:diag_obs_second}
            \E_{\setminus\ba^*}\Big[\frac{P^{\prime\prime}_{\rm out}(Y_{t\mu}\mid S_{t\mu})}{P_{\rm out}(Y_{t\mu}\mid S_{t\mu})}\mid \bW^*,\bv^*,\bX_\mu,\xi^*_\mu\Big]=0\,,
        \end{split}
    \end{align}
    and this allows us to center the $f_n=\log \cZ_t/n$ with its mean without changing the value of $C$. After using Cauchy-Schwartz we have
    \begin{align}
        \begin{split}
            C^2&\leq\mathbb{V}_{\setminus\ba^*}[f_n]\sum_{\mu,\nu=1}^{n}\E_{\setminus\ba^*}\Big[\E_{\setminus\ba^*}\Big[
            \frac{P^{\prime\prime}_{\rm out}(Y_{t\mu}\mid S_{t\mu})}{P_{\rm out}(Y_{t\mu}\mid S_{t\mu})}
            \frac{P^{\prime\prime}_{\rm out}(Y_{t\nu}\mid S_{t\nu})}{P_{\rm out}(Y_{t\nu}\mid S_{t\nu})}
            \mid  \bW^*,\bv^*,\bX_\mu,\bX_\nu,\xi^*_\mu,\xi^*_\nu\Big]\\
            &\qquad \times\Big(\epsilon-\frac{\Vert\varphi(\boldsymbol{\alpha}_\mu)\Vert^2-\rho\boldsymbol{\alpha}_\mu^\intercal\varphi(\boldsymbol{\alpha}_\mu)}{p}\Big)\Big(\epsilon-\frac{\Vert\varphi(\boldsymbol{\alpha}_\nu)\Vert^2-\rho\boldsymbol{\alpha}_\nu^\intercal \varphi(\boldsymbol{\alpha}_\nu)}{p}\Big)
            \Big] \,.
        \end{split}
    \end{align}
    Thanks to the observation in \eqref{eq:diag_obs_second}, only the diagonal terms $\mu=\nu$ will survive in the double sum on the r.h.s. of the previous inequality. Furthermore recall \eqref{ineq:P_out_xx}. Hence the bound on $C^2$ becomes
    \begin{align}
        \begin{split}
            C^2&\leq\mathbb{V}_{\setminus\ba^*}[f_n]C(f)n\E_{\bX_1,\bW^*}
            \Big(\epsilon-\frac{\Vert\varphi(\boldsymbol{\alpha}_1)\Vert^2-\rho\boldsymbol{\alpha}_1^\intercal\varphi(\boldsymbol{\alpha}_1)}{p}\Big)^2 \,.
        \end{split}
    \end{align}
    Following an integration by part and Lemma \ref{APPROX_LEMMA} we have
    \begin{align}
    \begin{split}
         \E_{\bW^*}\alpha_{1 i}\varphi(\alpha_{1 i})&= \E_{\bW^*}\varphi'(\alpha_{1 i})\frac{\Vert\bX_1\Vert^2}{d}= \frac{\Vert\bX_1\Vert^2}{d}\Big(\E_{\mathcal{N}(0,1)}\varphi'+O\Big(\frac{\Vert\bX_1\Vert^2}{d}-1\Big)\Big)\\
        &=\E_{\mathcal{N}(0,1)}\varphi'+O\Big( \frac{\Vert\bX_1\Vert^2}{d}-1\Big)\,.
    \end{split}
    \end{align}
    % \begin{align}
    %     \begin{split}
    %         \E_{\bW^*}^2\alpha_{\mu i}\varphi(\alpha_{\mu i})&= \E^2_{\bW^*}\varphi'(\alpha_{\mu i})\frac{\Vert\bX_\mu\Vert^4}{d^2}= \frac{\Vert\bX_\mu\Vert^4}{d^4}\Big(\E_{\mathcal{N}(0,1)}^2\varphi'+O\Big(\Big|\frac{\Vert\bX_\mu\Vert^2}{d}-1\Big|\Big)\Big)\\
    %         &=\frac{\Vert\bX_\mu\Vert^4}{d^2}\E_{\mathcal{N}(0,1)}^2\varphi'+O\Big(\Big|\frac{\Vert\bX_\mu\Vert^2}{d}-1\Big|\Big) \,.
    %     \end{split}
    % \end{align}
    From this and the approximation Lemma it follows that (letting $\E^2(\cdots)=(\E(\cdots))^2$)
    \begin{align}            \E_{\bX_{1}}\E_{\bW*}\Big[\frac{\Vert\varphi(\boldsymbol{\alpha}_1)\Vert^2-\rho\boldsymbol{\alpha}_1^\intercal\varphi(\boldsymbol{\alpha}_1)}{p}\Big]&=\E_{\mathcal{{N}}(0,1)}\varphi^2-\E^2_{\mathcal{{N}}(0,1)}\varphi'+O\Big(\E_{\bX_1}\Big|\frac{\Vert\bX_1\Vert^2}{d}-1\Big|\Big)\nonumber\\
    &=\epsilon+O(d^{-1/2})\,,\\
\E_{\bX_1}\E_{\bW*}\Big[\frac{\Vert\varphi(\boldsymbol{\alpha}_1)\Vert^2-\rho\boldsymbol{\alpha}_1^\intercal\varphi(\boldsymbol{\alpha}_1)}{p}\Big]^2&=\frac{1}{p}\E_{\bX_1}\E_{\bW^*}\big(
            \varphi^4(\alpha_{11})-2\rho\varphi^3(\alpha_{1 1})\alpha_{1 1}+\rho^2\alpha_{1 1}^2\varphi^2(\alpha_{11})
            \big)\nonumber\\
            &\hspace{-2cm}+\frac{p-1}{p}\E_{\bX_1}\big(\E_{\bW^*}^2\varphi^2(\alpha_{1 1})-2\rho\E_{\bW^*}\varphi^2(\alpha_{1 1})
            \E_{\bW^*}\varphi(\alpha_{11})\alpha_{11}+\rho^2\E_{\bW^*}^2\varphi(\alpha_{11})\alpha_{11}\big)\nonumber \\
&=\epsilon^2+O(p^{-1})+O(d^{-1/2})\,.
    \end{align}
    Hence, we finally have
    \begin{align}
        \begin{split}
            \E_{\bX_\mu,\bW^*}
            \Big(\epsilon-\frac{\Vert\varphi(\boldsymbol{\alpha}_\mu)\Vert^2-\rho\boldsymbol{\alpha}_\mu^\intercal\varphi(\boldsymbol{\alpha}_\mu)}{p}\Big)^2=O(p^{-1})+O(d^{-1/2})\,.
        \end{split}
    \end{align}
    Plugging this and the bound in Theorem \ref{th:Z_concentr} into the inequality for $C^2$ we readily get the statement.
\end{proof}

Now the remaining goal is to prove that $A_2-A_{12}\to0$. Using integration by parts in $A_2$ w.r.t. $\bv^*$ we obtain
\begin{align}
    \begin{split}
        A_2=\frac{\rho^2}{2n}\E_{(t)}\log \cZ_t\sum_{\mu,\nu=1}^{n}U_{\mu\nu}\frac{\bX_\mu^\intercal\bX_\nu }{d}\,.
    \end{split}
\end{align}
Recall also formula \eqref{A12} for $A_{12}$.

\begin{lemma}[$A_{12}-A_2$ term]
    The following asymptotics holds:
    \begin{align}\label{lemmalhs_star}
        A_{12}-A_2=\frac{\rho}{2n}\E_{(t)}\log \cZ_t\sum_{\mu,\nu=1}^{n}U_{\mu\nu}\frac{\bX_\mu^\intercal\bX_\nu }{d}\Big[\frac{\mathbf{a}^{*\intercal}(\mathbf{a}^*\circ\varphi'(\boldsymbol{\alpha}_\nu))}{p}-{\rho}\Big]=O\Big(\sqrt{\Big( 1+\frac nd\Big)\Big(\frac{n}{dp}+\frac{n}{d^{3/2}}\Big)} \Big)\,.
    \end{align}
\end{lemma}

\begin{proof}
Conditional on $\ba^*$ define
\begin{align}
    C:=\frac{1}{n}\E_{\setminus\ba^*}\log \cZ_t\sum_{\mu,\nu=1}^{n}U_{\mu\nu}\frac{\bX_\mu^\intercal\bX_\nu }{d}\Big[\frac{\mathbf{a}^{*\intercal}(\mathbf{a}^*\circ\varphi'(\boldsymbol{\alpha}_\nu))}{p}-{\rho}\Big]\,.
\end{align}
As before, thanks to Lemma \ref{lem:propertiesPout}, we can center the random variable $\log \cZ_t$ with its expectation $\E_{\setminus \ba^*}\log \cZ_t$, without affecting $C$. We can thus use Cauchy-Schwartz's inequality, obtaining
\begin{align}
    C^2\leq  \mathbb{V}_{\setminus\ba^*}[f_n]\E_{\setminus\ba^*}\sum_{\mu,\nu=1}^{n}\sum_{\lambda,\eta=1}^{n}U_{\mu\nu}U_{\lambda\eta}\frac{\bX_\mu^\intercal\bX_\nu }{d}\Big[\frac{\mathbf{a}^{*\intercal}(\mathbf{a}^*\circ\varphi'(\boldsymbol{\alpha}_\nu))}{p}-{\rho}\Big]\frac{\bX_\lambda^\intercal\bX_\eta }{d}\Big[\frac{\mathbf{a}^{*\intercal}(\mathbf{a}^*\circ\varphi'(\boldsymbol{\alpha}_\eta))}{p}-{\rho}\Big]
\end{align}
for a given $\ba^*$. Thanks again to Lemma \ref{lem:propertiesPout} the only terms that survive in the above quadruple sum are those with $\mu=\nu=\lambda=\eta$, and $\mu\neq \nu$, $\lambda\neq \eta$ but with $\mu=\lambda$, $\nu=\eta$ or vice versa. Up to constants everything can be summed up as follows:
\begin{align}
    C^2\leq  K \mathbb{V}_{\setminus\ba^*}[f_n]\E_{\setminus\ba^*}\sum_{\mu,\nu=1}^{n}\Big(\frac{\bX_\mu^\intercal\bX_\nu }{d}\Big)^2\Big[\frac{\mathbf{a}^{*\intercal}(\mathbf{a}^*\circ\varphi'(\boldsymbol{\alpha}_\nu))}{p}-{\rho}\Big]^2
\end{align}where we used again Lemma \ref{lem:propertiesPout}. Now, expanding the square and computing the $\bW^*$ average of $\varphi'$ via Lemma \eqref{APPROX_LEMMA} we readily get
\begin{align}
    C^2\leq K' \mathbb{V}_{\setminus\ba^*}[f_n]\sum_{\mu,\nu=1}^{n}\E_{\bX_\mu,\bX_\nu}\Big(\frac{\bX_\mu^\intercal\bX_\nu }{d}\Big)^2\Big[\Big(\frac{\Vert\ba^*\Vert^2}{p }-1\Big)^2+O\Big(\frac{\Vert\bX_\nu\Vert^2}{d}-1\Big) \Big(\frac{\Vert\ba^*\Vert^4}{p^2}+\frac{\Vert\ba^*\Vert^2}{p}\Big)\Big]
\end{align}where $K'$ is a suitable positive constant.
Denoting the double sum by $D$ we have
\begin{align}
    |A_{2}-A_{12}|\leq  K''\E_{\ba^*}\sqrt{\mathbb{V}_{\setminus\ba^*}[f_n] }\sqrt D\leq K''\sqrt{\E_{\ba^*}\mathbb{V}_{\setminus\ba^*}[f_n]\,\E_{\ba^*}D}=O\Big(\sqrt{\Big(\frac 1n+\frac 1d\Big)\Big(\frac{n^2}{dp}+\frac{n^2}{d^{3/2}}\Big)} \Big)\,.
\end{align}
\end{proof}

Putting the results of all the Lemmas in this section together we get that the time derivative of the interpolating free entropy is bounded by
\begin{align}\label{eq:gathering_orders}
    \frac d{dt}{\bar f}_n(t)&=\underbrace{O\Big(\sqrt{\Big(1+\frac{n}{d}\Big)\Big(\frac{n}{p}+\frac{n}{d^{3/2}}\Big)}\Big)}_{A_{11}^{\rm off}}+ \underbrace{O\Big(\sqrt{\Big(1+\frac n d \Big)\Big(\frac1p + \frac1{\sqrt d}\Big)}\Big)}_{A_3-A_{11}^{\rm diag}}+\underbrace{O\Big(\sqrt{\Big(1 +\frac nd\Big)\Big(\frac{n}{dp}+\frac{n}{d^{3/2}}\Big)} \Big)}_{ A_{12}-A_2}\\
    &=O\Big(\sqrt{\Big(1+\frac{n}{d}\Big)\Big(\frac{n}{p}+\frac{n}{d^{3/2}}+\frac1{\sqrt d}\Big)}\Big)\,.
\end{align}
All the bounds in this section are uniform in $t\in[0,1]$. This finishes the proof of Theorem \ref{thm:equivalence}. 
\begin{comment}
    The last convergence speed contains subleading terms w.r.t. to the other rates, or terms that are already taken into account, in fact:
\begin{align}
    \begin{split}
        O\Big(\sqrt{\Big(\frac 1n+\frac 1d\Big)\Big(\frac{n^2}{dp}+\frac{n^2}{d^{3/2}}\Big)} \Big=O\Big(\sqrt{\Big(1+\frac nd\Big)\Big(\frac{n}{dp}+\frac{n}{d^{3/2}}\Big)} \Big)\,.
    \end{split}
\end{align}
The $\Big(1+\frac{n}{d}\Big)\frac{n}{d^{3/2}}$ is already contained in other rates, whereas the term coming from $n/(pd)$ is sub-leading w.r.t. $n/p$, also present in other rates. Concerning instead the middle contribution in \eqref{eq:gathering_orders}, the part coming from
\begin{align}
    \frac{n}{d}\Big(\frac{1}{p}+\frac{1}{\sqrt{d}}\Big)=\frac{n}{pd}+\frac{n}{d^{3/2}}
\end{align}

\end{comment}

\section{Concentration of the free entropy}\label{sec:concentration}

Here we prove that the free entropy of the interpolating model 
concentrates, i.e., Theorem~\ref{th:Z_concentr}. To simplify the notations we use $C(f,\varphi)$ for a generic non-negative
constant depending only on $f$ and $\varphi$. We recall that the partition function is defined  as 
\begin{align}
    \cZ_t=\int D\mathbf{a} D\mathbf{v}D\mathbf{W}\prod_{\mu=1}^{n}D\boldsymbol{\xi}_{\mu}\exp\Big[\sum_{\mu=1}^{n}\log P_{\rm out}(Y_{t\mu}|s_{t\mu})\Big]\,,
\end{align}
where
\begin{align}
    Y_{t\mu}&=f(S_{t\mu};\bA_\mu)+\sqrt{\Delta}Z_\mu\,,\\
 S_{t\mu}&=\sqrt{1-t}\frac{\mathbf{a}^{*\intercal}}{\sqrt{p}}\varphi\Big(\frac{\mathbf{W}^*\mathbf{X}_\mu}{\sqrt{d}}\Big)+\sqrt{t\rho}\frac{\mathbf{v}^{*\intercal}\mathbf{X}_\mu}{\sqrt{d}}+\sqrt{t\epsilon}\xi_\mu^*\,,\\
    s_{t\mu}&=\sqrt{1-t}\frac{\mathbf{a}^\intercal}{\sqrt{p}}\varphi\Big(\frac{\mathbf{W}\mathbf{X}_\mu}{\sqrt{d}}\Big)+\sqrt{t\rho}\frac{\mathbf{v}^\intercal\mathbf{X}_\mu}{\sqrt{d}}+\sqrt{t\epsilon}\xi_\mu\,.
\end{align}
% One can see that $Z_t$ is random through $Z_\mu$, $\xi^*_\mu$, $\mathbf{X}_\mu$, $\bA_\mu$,  $\mathbf{W}^*$ and $\bv^*$. 

We  prove  Theorem~\ref{th:Z_concentr} in several steps, first we show concentration   with respect only to $\{Z_\mu\}_{\mu}$, $\{\xi^*_\mu\}_{\mu}$ and $\{\mathbf{X}_\mu\}_{\mu}$ using  classical Poincare-Nash inequality, then with respect to $\{\bA_\mu\}_{\mu}$ using the corollary of Efron-Stein inequality, and then finally with respect to $\mathbf{W}^*$ with $\bv^*$. For this we rewrite
%, only after this we take the final expectation over $\mathbf{a}^*$. 
\begin{multline}
     \mathbb{E}\Big(\frac{1}{n}\log \cZ_t-\mathbb{E}_{\bv^*,\mathbf{W}^*,\mathbf{X},\boldsymbol{\xi}^*,\bA,\bZ}\frac{1}{n}\log \cZ_t\Big)^2=\\
    % \mathbb{E}\Big(\frac{1}{n}\log \cZ_t-\mathbb{E}_{\mathbf{X},\xi^*,\bZ}\frac{1}{n}\log \cZ_t+\mathbb{E}_{\mathbf{X},\xi^*,\bZ}\frac{1}{n}\log \cZ_t-\mathbb{E}_{\mathbf{X},\xi^*,\bZ,\bA}\frac{1}{n}\log \cZ_t+\mathbb{E}_{\mathbf{X},\xi^*,\bZ,\bA}\frac{1}{n}\log \cZ_t-\mathbb{E}_{\mathbf{v}^*,\mathbf{W}^*,\mathbf{X},\xi^*,\bZ,\bA}\dfrac{1}{n}\log \cZ_t\Big)^2\\
    = \mathbb{E}\Big(\frac{1}{n}\log \cZ_t-\mathbb{E}_{\mathbf{X},\boldsymbol{\xi}^*,\bZ}\frac{1}{n}\log \cZ_t\Big)^2
    +\mathbb{E}\Big(\frac{1}{n}\E_{\mathbf{X},\boldsymbol{\xi}^*,\bZ}\log \cZ_t-\E_{\bA}\mathbb{E}_{\mathbf{X},\boldsymbol{\xi}^*,\bZ}\frac{1}{n}\log \cZ_t\Big)^2+\\
    +\mathbb{E}\Big(\mathbb E_{\psi}\frac{1}{n}\log \cZ_t-\mathbb{E}_{\mathbf{v}^*,\mathbf{W}^*}\mathbb E_{\psi}\frac{1}{n}\log \cZ_t\Big)^2
\end{multline}
where by $\mathbb{E}_{\psi}$ we denoted the joint expectation with respect to $\bZ$, $\boldsymbol{\xi}^*$, $\mathbf{X}$, and $\bA$. Also for brevity, in what follows,  by writing $\bZ$, $\boldsymbol{\xi}^*$, $\mathbf{X}$, and $\bA$ we mean the sets $\{Z_\mu\}_{\mu}$, etc.
We recall two classical concentration results, whose proofs can be found in \cite{Boucheron2004}, Chapter 3.
\begin{proposition}[Poincaré-Nash inequality]
  \label{prop:poincare}
	Let $\xi=[\xi_1,\ldots,\xi_K]^\intercal$ be a real Gaussian standard random vector. If $g:\R^K\rightarrow \R$ is a continuously  differentiable function, then
\begin{align}\label{ineq:p-n}
	\V g(\xi)\le\E\|\nabla g(\xi)\|^2\,.
\end{align}
\end{proposition}

\begin{proposition}[Bounded difference]
 \label{prop:efron-stein}   
 Let $\xi=[\xi_1,\ldots,\xi_K]^\intercal$ be a random vector with i.i.d. elements taking values in some space $\mathcal{A}$. If function $g:\mathcal{A}^K\rightarrow \R$ satisfies 
 \begin{align}
     \sup_{1\leq i\leq K}\sup_{x_1,\ldots,x_K, x\prime_i\in \mathcal{A}}|g(x_1,\ldots,x_i,\ldots, x_K)-g(x_1,\ldots,x_i^\prime,\ldots, x_K)|
     \leq C
 \end{align} for some $C>0$, then 
\begin{align}\label{ineq:e-s}
	\V\{g(\xi)\}\le\dfrac{1}{4}KC^2.
\end{align}
\end{proposition}

In what follows we will denote $P^y(y|x):=\frac{\partial P_{\rm out}(y|x)}{\partial y}$ and $P^x(y|x):=\frac{\partial P_{\rm out}(y|x)}{\partial x}$.
\begin{lemma}
  There exists a non-negative constant $C(f,\varphi$) such that
  \begin{align*}
      \mathbb{E}\Big(\frac{1}{n}\log \cZ_t-\mathbb{E}_{\mathbf{X},\boldsymbol{\xi}^*,\bZ}\frac{1}{n}\log \cZ_t\Big)^2\leq \dfrac{C(f,\varphi)}{n}\,.
  \end{align*}
\end{lemma}
\begin{proof}
    Since $\xi_{\mu}^*$, $Z_\mu$ and all elements of vectors $\mathbf{X}_{\mu}$ are jointly independent for all $\mu$,  we have thanks to Proposition~\ref{prop:poincare}
    \begin{multline*}
        \E\V_{\mathbf{X},\boldsymbol{\xi}^*,\bZ}\Big(\frac{1}{n}\log \cZ_t\Big)\leq \frac{1}{n^2}\E\|\nabla \log \cZ_t\|^2\\
        =\frac{1}{n^2}\sum_{\mu=1}^{n}\E\Big(\frac{\partial \log \cZ_t}{\partial\xi^*_{\mu}}\Big)^2
        +\frac{1}{n^2}\sum_{\mu=1}^{n}\sum_{i=1}^{d}\E\Big(\frac{\partial \log \cZ_t}{\partial X^i_{\mu}}\Big)^2
        +\frac{1}{n^2}\sum_{\mu=1}^{n}\E\Big(\frac{\partial \log \mathcal{Z}_t}{\partial Z_{\mu}}\Big)^2=:I_1+I_2+I_3\,.
    \end{multline*}
For the sake of brevity we drop index $\rm out$ and write $P_{\mu}=P_{\rm out}(Y_{t,\mu}|s_{t,\mu})$ in what follows.  Gibbs brackets $\langle\cdot\rangle$ are defined as  in (\ref{def:gibbs_brack}).
After taking derivative in the first term we obtain
\begin{align*}
    \Big|\frac{\partial \log \cZ_t}{\partial\xi^*_{\mu}}\Big|=\Big|\Big\langle\frac{P_{\mu}^y}{P_{\mu}}\Big\rangle f^\prime(S_{t,\mu};\bA_\mu)\sqrt{t\epsilon}\Big|
\leq c\sqrt{t\epsilon}C(f)(|Z_\mu|^2+1)\,.
\end{align*}
    Last inequality is due to boundedness of $f^\prime$  and Lemma~\ref{lem:bound_P_out_der}. One can see that the only randomness left is in $Z_\mu$. Since it is Gaussian, the average of polynomial is bounded by some uniform constant $c$. We obtained that each term in $I_1$ is bounded by constant, the number of terms is $n$, from this it follows immediately that $I_1\leq C(f)/n$.
    
The second type of partial derivative will give us
\begin{multline*}
    \frac{\partial \log \cZ_t}{\partial X^i_{\mu}}=\Big\langle\frac{P_{\mu}^y}{P_{\mu}}\Big\rangle f^\prime(S_{t,\mu};\bA_\mu)\Big(\sqrt{1-t}\frac{\big(\mathbf{a}^*\circ\varphi^\prime(\frac{\mathbf{W}^*\mathbf{X}_{\mu}}{\sqrt{d}})\mathbf{W}^*\big)_i}{\sqrt{pd}}+\sqrt{t\rho}\frac{v^*_i}{\sqrt{d}}\Big)\\
    +\Big\langle\frac{P_{\mu}^x}{P_{\mu}} \frac{\big(\mathbf{a}\circ\varphi^\prime(\frac{\mathbf{W}\mathbf{X}_{\mu}}{\sqrt{d}})\mathbf{W}\big)_i}{\sqrt{pd}}\Big\rangle \sqrt{1-t}
    +\Big\langle\frac{P_{\mu}^x}{P_{\mu}} \frac{v_i}{\sqrt{d}}\Big\rangle \sqrt{t\rho}\,.
\end{multline*}
   Plugging this into $I_2$ and using the simple inequality $\E(a+b)^2\leq2(\E a^2+\E b^2)$ in order to square each term of the r.h.s. separately, we notice that it appeared the terms  $$\E\Big\langle\frac{P_{\mu}^x}{P_{\mu}} \frac{\big(\mathbf{a}\circ\varphi^\prime(\frac{\mathbf{W}\mathbf{X}_{\mu}}{\sqrt{d}})\mathbf{W}\big)_i}{\sqrt{pd}}\Big\rangle^2\,, \qquad \E   \Big\langle\frac{P_{\mu}^x}{P_{\mu}} \frac{v_i}{\sqrt{d}}\Big\rangle^2\,,$$ which depend only on $Y_{t\mu}, S_{t\mu}$ and $s_{t\mu}$. This allow us to use Nishimori identity by removing the brackets and adding $*$ to $\ba$, $\bW$ and $\bv$. Then evaluating each ratio with $P_{\rm out}$ using Lemma~\ref{lem:bound_P_out_der} we get 
    \begin{align}
        I_2\leq \frac{C(f) (1-t)}{n^2}\sum_{\mu=1}^{n}\E\Big((|Z_\mu|^2+1)^2\frac{\|\mathbf{a}^*\circ\varphi^\prime\Big(\frac{\mathbf{W}^*\mathbf{X}_{\mu}}{\sqrt{d}}\Big)\mathbf{W}^*\|^2}{pd}\Big)
        +\dfrac{C(f)t\rho}{n}\E\Big((|Z_\mu|^2+1)^2\dfrac{\|\bv^*\|^2}{d}\Big)\,.
    \end{align}
     after what we notice that factors $(|Z_\mu|^2+1)^2$ are independent of others and its expectation can be bounded with positive constant. In the end we obtain
        \begin{align}
        I_2\leq \frac{C( f)}{n^2}\sum_{\mu=1}^{n}\E\frac{\|\mathbf{a}^*\circ\varphi^\prime\big(\frac{\mathbf{W}^*\mathbf{X}_{\mu}}{\sqrt{d}}\big)\mathbf{W}^*\|^2}{pd}
        +\dfrac{C( f)}{n}\,.
    \end{align}
   To finish the proof we notice that when we expand the norm appearing above all non zero terms will contain only squares, e.g., $\E(a^*_i\varphi^\prime_iW^*_{ij})^2$,  and so be positive. This gives as opportunity to bound $\varphi^\prime$ with $C(\varphi)$ in each term and calculate its expectation which is simply $C(\varphi)$. The number of such terms is exactly $pd$, this gives us simple bound 
  $I_2\leq C(f,\varphi)/n$.
 
  Finally, the derivatives with respect to $Z_\mu$ are of the form
  \begin{align}
 \frac{\partial \log \cZ_t}{\partial Z_{\mu}}=\sqrt{\Delta}\Big\langle\frac{P^y_\mu}{P_\mu}\Big\rangle\,.
  \end{align}
  Similarly to what done above we bound $I_3$ with the help of Lemma~\ref{lem:bound_P_out_der}
  \begin{align}
      I_3\leq \frac{C(f)}{n}\,.
  \end{align}
  
\end{proof}

Next step would be to prove the concentration of function $\E_{\mathbf{X},\bZ,\boldsymbol{\xi}^*}\log \cZ_t/n$ with respect to  $\bA$ using Proposition~\ref{prop:efron-stein} while keeping $\mathbf{W}^*$ and $\mathbf{v}^*$ fixed.

\begin{lemma}
      There exists a constant $C(f,\varphi)>0$ such that 
    \begin{align}
        \E\Big(\frac{1}{n}\E_{\mathbf{X},\bZ,\boldsymbol{\xi}^*}\log \cZ_t-\E_{\bA}\frac{1}{n}\E_{\mathbf{X},\bZ,\boldsymbol{\xi}^*}\log \cZ_t\Big)^2\leq \frac{C(f,\varphi)}{n}\,.
    \end{align}  
\end{lemma}
\begin{proof}
    We consider $h(\bA)=\E_{\mathbf{X},\bZ,\boldsymbol{\xi}^*}\log \cZ_t/n$ a function of all the elements $\bA_{\mu,i}$ of $\bA_\mu$ for $1\le i\le k$ and $1\leq\mu\leq n$. We  denote by $\bA^\prime$ a vector such that $\bA^\prime_{\mu,i}=\bA_{\mu,i}$ for $\mu\neq\nu$, $i\neq j$ and $\bA^\prime_{\nu,j}$ is a random variable with distribution $P_A$, independent of all others.
    According to Proposition~\ref{prop:efron-stein} it is sufficient to prove that 
    \begin{align}\label{ineq:bound_dif_h}
        |h(\bA^\prime)-h(\bA)|<\frac{C(f,\varphi)}n\,.
    \end{align}
    If we denote by $H$ (and $H^\prime$) the Hamiltonian corresponding to $\cZ_t$ (and $\cZ_t$ with $\bA^\prime$) one can see that
    \begin{align}
        h(\bA^\prime)-h(\bA)
        =\frac{1}{n}\E_{\mathbf{X},\bZ,\boldsymbol{\xi}^*}\log \Big\langle e^{H-H^\prime}\Big\rangle_{H}
        \geq \frac{1}{n}\E_{\mathbf{X},\bZ,\boldsymbol{\xi}^*} \Big\langle H-H^\prime\Big\rangle_{H}\,,
    \end{align}
    the last step being true due to Jensen inequality. On the other hand $h(\bA^\prime)-h(\bA)
                \leq \E_{\mathbf{X},\bZ,\boldsymbol{\xi}^*} \langle H-H^\prime\rangle_{H^\prime}/n$.
  We recall the definition (\ref{eq:Pout}) of $P_{\rm out}(Y_{t\mu}, s_{t\mu}) $ and similarly we obtain
  \begin{align}
     H-H^\prime\geq\frac{1}{2\Delta}\Big\langle (f(S_{t\nu};\bA_\nu^\prime)-f(s_{t\nu};\tilde{\bA})+\sqrt{\Delta}Z_\nu)^2-(f(S_{t\nu};\bA_\nu)-f(s_{t\nu};\tilde{\bA})+\sqrt{\Delta}Z_\nu)^2\Big\rangle_{G^\prime}
\end{align}
and
  \begin{align}
     H-H^\prime\leq\frac{1}{2\Delta}\Big\langle (f(S_{t\nu};\bA_\nu^\prime)-f(s_{t\nu};\tilde{\bA})+\sqrt{\Delta}Z_\nu)^2-(f(S_{t\nu};\bA_\nu)-f(s_{t\nu};\tilde{\bA})+\sqrt{\Delta}Z_\nu)^2\Big\rangle_{G}\,,
\end{align}
where $\langle\cdot\rangle_G$ (or with $G^\prime$) defined as
\begin{align}
   \langle\cdot\rangle_G=\dfrac{\int P_A(d\tilde{\bA})e^{-\frac{1}{2\Delta}(Y_{t\nu}-f(s_{t\nu};\tilde{\bA}))^2}(\cdot)}{\int P_A(d\tilde{\bA})e^{-\frac{1}{2\Delta}(Y_{t\nu}-f(s_{t\nu};\tilde{\bA}))^2}} 
\end{align}
or with $Y^\prime_{t\nu}$ where $\bA_\nu$ is changed to $\bA^\prime_\nu$. Since $f$ is bounded we immediately obtain $|H-H^\prime|\leq C(f)(|Z_\mu|^2+1)$ and (\ref{ineq:bound_dif_h}). Now, due to the Proposition~\ref{prop:efron-stein}, the statement of the Lemma is proved.

\end{proof}

The last part is to prove the concentration 
of function $g=\E_{\psi}\log \cZ_t/n$ with respect to $\mathbf{W}^*,\bv^*$.

\begin{lemma}
    There exists a constant $C(f,\varphi)>0$ such that 
    \begin{align}
        \E(g-\E_{\mathbf{W}^*,\bv^*}g)^2\leq \frac{C(f,\varphi)}{d}\,.
    \end{align}
\end{lemma}
\begin{proof}
    Due to Poincare-Nash inequality we have 
    \begin{align}\label{ineq:PN_E_psi}
         \E(g-\E_{\mathbf{W}^*,\bv^*}g)^2\leq \sum_{i,j}^{p,d}\E\Big(\frac{\partial g}{\partial W^*_{ij}}\Big)^2
         +\sum_{i}^{d}\E\Big(\frac{\partial g}{\partial v^*_{i}}\Big)^2\,.
\end{align}
Let us first deal with the partial derivatives with respect to $W^*_{ij}$
        \begin{align} 
         \frac{\partial g}{\partial W^*_{ij}}
         =\frac{1}{n}\sum_{\mu}^{n}\E_{\psi}\Big(\Big\langle\frac{P_{\mu}^y}{P_{\mu}}\Big\rangle f^\prime(S_{t\mu},\bA_\mu)\frac{\sqrt{1-t}a_i^*\varphi^\prime_iX_{\mu}^j}{\sqrt{pd}}\Big)\,,
    \end{align}
    where $\varphi^\prime_i=\varphi^\prime(\mathbf{W}^*_i\mathbf{X}_{\mu}/\sqrt{d})$. 
    Before integrating by parts with respect to $X_{\mu}^j$  let us notice that in the sum over $\mu$ all terms are the same since we are taking the expectation over all i.i.d. vectors $\mathbf{X_{\mu}}$, $\bA_\mu$, $Z_\mu$ and $\xi_{\mu}$, it means that we can disregard the sum and just multiply by $n$ directly.
    Then we have
    \begin{multline}
         \frac{\partial g}{\partial W^*_{ij}}
         =\E_{\psi}\Big[\frac{\partial S_{t\mu}}{\partial X^j_{\mu}}\frac{\sqrt{1-t}a_i^*\varphi^\prime_{i}}{\sqrt{pd}}\Big(-\Big\langle\frac{P_{\mu}^y}{P_{\mu}}\Big\rangle^2f^\prime(S_{t\mu};\bA_\mu)^2+\Big\langle\frac{P_{\mu}^{yy}}{P_{\mu}}\Big\rangle f^\prime(S_{t\mu};\bA_\mu)^2+\Big\langle\frac{P_{\mu}^y}{P_{\mu}}\Big\rangle f^{\prime\prime}(S_{t\mu};\bA_\mu)\Big)\Big]\\
        +\E_{\psi}\Big[\frac{\sqrt{1-t}a_i^*\varphi^\prime_{i}}{\sqrt{pd}}\Big(-\Big\langle\frac{P_{\mu}^y}{P_{\mu}}\Big\rangle \Big\langle\frac{P_{\mu}^{x}}{P_{\mu}}\frac{\partial s_{t\mu}}{\partial X^j_{\mu}}\Big\rangle +\Big\langle\frac{P_{\mu}^{yx}}{P_{\mu}}\frac{\partial s_{t\mu}}{\partial X^j_{\mu}}\Big\rangle \Big)f^\prime(S_{t\mu};\bA_\mu)\Big]
        +\E_{\psi}\Big[\frac{\sqrt{1-t}a_i^*\varphi^{\prime\prime}_{i}W^*_{ij}}{\sqrt{p}d}f^\prime(S_{t\mu};\bA_\mu)\Big]\,.\nonumber
    \end{multline}
    Due to the Lemma~\ref{lem:bound_P_out_der}  absolute values of ratios of $P$'s are bounded with $C(f)(|Z_\mu|^2+1)$. In the first term of expression above one can easily get rid of $Z_\mu$ since  $\E_\psi\langle(1+|Z_\mu|^2)\rangle<C$. On the other hand   derivatives of $f$ and $\varphi$, in view of (\nameref{Assum:f}, remain bounded with non-negative constant $C(f,\varphi)$, so combining all above and plugging into latter expression along with derivatives of $S_{t\mu}$ and $s_{t\mu}$ we obtain
    \begin{multline}
        \Big| \frac{\partial g}{\partial W^*_{ij}}\Big|
        \leq \frac{C(f,\varphi) }{\sqrt{d}}\E_{\psi}\Big|\frac{\big(\mathbf{a}^*\circ\varphi^\prime(\frac{\mathbf{W}^*\mathbf{X}_{\mu}}{\sqrt{d}})\mathbf{W}^*\big)_j}{\sqrt{pd}}\frac{a_i^*}{\sqrt{p}}\Big|
        +\frac{C(f,\varphi) }{\sqrt{d}}\E_{\psi}\Big|\frac{v^*_j}{\sqrt{d}}\frac{a_i^*}{\sqrt{p}}\Big|\\
        +\frac{C(f,\varphi) }{\sqrt{d}}\E_{\psi}\Big|(|Z_\mu|^2+1)\frac{a_i^*}{\sqrt{p}}\Big\langle\frac{\big(\mathbf{a}\circ\varphi^\prime(\frac{\mathbf{W}\mathbf{X}_{\mu}}{\sqrt{d}})\mathbf{W}\big)_j}{\sqrt{pd}}\Big\rangle\Big|
        +\frac{C(f,\varphi) }{\sqrt{d}}\E_{\psi}\Big|(|Z_\mu|^2+1)\frac{a_i^*}{\sqrt{p}}\Big\langle\frac{v_j}{\sqrt{d}}\Big\rangle\Big|
        +\frac{C(f,\varphi)}{\sqrt{d}}\E_{\psi}\Big|\frac{a_i^*W^*_{ij}}{\sqrt{pd}}\Big|\,.\nonumber
    \end{multline}
After  using repeatedly $\E(a+b)^2\leq 2\E a^2+2\E b^2$ along with Jensen inequality one can show
\begin{multline}\label{ineq:sum_part_W}
    \sum_{i,j}^{p,d}\E \Big| \frac{\partial g}{\partial W^*_{ij}}\Big|^2
        \leq \frac{C(f,\varphi) }{d}\Big(\E\Big[\frac{\|\mathbf{a}^*\circ\varphi^\prime(\frac{\mathbf{W}^*\mathbf{X}_{\mu}}{\sqrt{d}})\mathbf{W}^*\|^2}{pd}\frac{\|\ba^*\|^2}{p}\Big]\\
        +\E\Big[\frac{(|Z_\mu|^2+1)^2\|\ba^*\|^2}{p}\Big\langle\frac{\|\mathbf{a}\circ\varphi^\prime(\frac{\mathbf{W}\mathbf{X}_{\mu}}{\sqrt{d}})\mathbf{W}\|^2}{pd}\Big\rangle\Big]
        +\E\Big[\frac{(|Z_\mu|^2+1)^2\|\ba^*\|^2}{p}\Big\langle\frac{\|\bv\|^2}{d}\Big\rangle\Big]
        +2\Big)\,.
\end{multline}
Two terms of the form $\E[b\langle c\rangle]$ we bound by using Cauchy-Schwartz inequality, Jensen's inequality and Nishimori identity consecutively:
\begin{align}
    \E[b\langle c\rangle]\leq \E^{1/2}[b^2]\E^{1/2}[\langle c\rangle^2]
    \leq \E^{1/2}[b^2]\E^{1/2}[\langle c^2\rangle]
    \leq \E^{1/2}[b^2]\E^{1/2}[ c^2]\,.
\end{align}
This allows us to rewrite (\ref{ineq:sum_part_W}) as
\begin{align}
    \sum_{i,j}^{p,d}\E \Big| \frac{\partial g}{\partial W^*_{ij}}\Big|^2
        \leq \frac{C(f,\varphi) }{d}\Big(\E^{1/2}\Big[\frac{\|\mathbf{a}^*\circ\varphi^\prime(\frac{\mathbf{W}^*\mathbf{X}_{\mu}}{\sqrt{d}})\mathbf{W}^*\|^4}{p^2d^2}\Big]
        +C\Big)\,.
\end{align}
What is left is  to notice that in $\E\|\mathbf{a}^*\circ\varphi^\prime(\mathbf{W}^*\mathbf{X}_{\mu}/ {\sqrt{d}})\mathbf{W}^*\|^4/(p^2d^2)$ all non zero terms will have only even powers so we can bound $\varphi^\prime$ with a constant in all of them, which gives immediately 
   \begin{align}\label{ineq:final_bound_W}
   \sum_{i,j}^{p,d}\E \Big| \frac{\partial g}{\partial W^*_{ij}}\Big|^2
        \leq \frac{C(f,\varphi) }{d}\,.
\end{align}

Now we consider the partial derivative with respect to $v_i^*$
\begin{align}
    \frac{\partial g}{\partial v^*_i}=\frac{1}{n}\sum_\mu^{n}\E_\psi\Big[\Big\langle\frac{P^y_\mu}{P_\mu}\Big\rangle f^\prime(S_{t,\mu};\bA_\mu)\frac{\sqrt{t\rho} X_\mu^i}{\sqrt{d}}\Big]\,.
\end{align}
As in the case with $W_{ij}^*$ it is necessary to integrate by parts also with respect to $X_\mu^i$ since blind bounds will not give us the desired order. The result will be very similar to the previous calculation just this time we don't have factors $\varphi^\prime_i$ and $a^*_i$. After similar simplification (bounds on ratios of $P$'s, etc.) we obtain
 \begin{multline}
        \Big| \frac{\partial g}{\partial v^*_{i}}\Big|
        \leq \frac{C(f,\varphi) }{\sqrt{d}}\E_{\psi}\Big|\frac{(\mathbf{a}^*\circ\varphi^\prime(\frac{\mathbf{W}^*\mathbf{X}_{\mu}}{\sqrt{d}})\mathbf{W}^*)_i}{\sqrt{pd}}\Big|
        +\frac{C(f,\varphi) }{\sqrt{d}}\E_{\psi}\Big|\frac{v^*_i}{\sqrt{d}}\Big|\\
        +\frac{C(f,\varphi) }{\sqrt{d}}\E_{\psi}\Big|(|Z_\mu|^2+1)\Big\langle\frac{(\mathbf{a}\circ\varphi^\prime(\frac{\mathbf{W}\mathbf{X}_{\mu}}{\sqrt{d}})\mathbf{W})_i}{\sqrt{pd}}\Big\rangle\Big|
        +\frac{C(f,\varphi) }{\sqrt{d}}\E_{\psi}\Big|(|Z_\mu|^2+1)\Big\langle\frac{v_i}{\sqrt{d}}\Big\rangle\Big|\,.
    \end{multline}
Similarly to the previous case it is easy to see that

 \begin{multline}
        \sum_i^{d}\E\Big| \frac{\partial g}{\partial v^*_{i}}\Big|^2
        \leq \frac{C(f,\varphi) }{d}\Big(\E\Big[\frac{\|\mathbf{a}^*\circ\varphi^\prime(\frac{\mathbf{W}^*\mathbf{X}_{\mu}}{\sqrt{d}})\mathbf{W}^*\|^2}{pd}\Big]
        +\E\Big[\frac{\|\bv^*\|^2}{d}\Big]\\
+\E\Big[(|Z_\mu|^2+1)^2\Big\langle\frac{\|\mathbf{a}\circ\varphi^\prime(\frac{\mathbf{W}\mathbf{X}_{\mu}}{\sqrt{d}})\mathbf{W}\|^2}{pd}\Big\rangle\Big]
        +\E\Big[(|Z_\mu|^2+1)^2\Big\langle\frac{\|\bv\|^2}{d}\Big\rangle\Big]\Big)\,.
    \end{multline}
    In order to be able to apply Nishimori identity to the last two terms we have first to  use Cauchy-Shchwarz inequality to separate factors with noise $Z_\mu$ from the Gibbs bracket.  Due to the fact that $\varphi^\prime$ is bounded we obtain  
 \begin{align}
        \sum_i^{d}\E\Big| \frac{\partial g}{\partial v^*_{i}}\Big|^2
        \leq \frac{C(f,\varphi) }{d}\,.
    \end{align}
    This combined with (\ref{ineq:final_bound_W}) finishes the proof. 
\end{proof}

\section{Proof of Proposition \ref{prop:gen_error}}\label{sec:gen_error}
The proof is implied by that of Theorem \ref{thm:equivalence}. We introduce a further set of data $\tilde{\mathcal{D}}_n:=\{(\bX_\nu,\tilde{Y}_{\nu})\}_{n+1\leq\nu\leq n(1+\varepsilon)}$ with responses
\begin{align}\label{eq:sideinfo}
    \tilde Y_{\nu}=\sqrt{\lambda}Y'_{\nu}+\tilde Z'_\nu\,,\quad  Y'_{\nu}\sim P_{\rm out}(\cdot\mid S_{\nu})
\end{align}
where $\lambda\geq 0$ and $n+1\leq\nu\leq n(1+\varepsilon)$ for some $\varepsilon\geq 0$, $\tilde Z_\nu'$ are i.i.d. Gaussian variables independent of everything else, and where $S_\nu$ is defined as in \eqref{7} but for the new inputs, with same teacher as used to generate $\mathcal{D}_n$. We now define a proxy for the Bayes-optimal generalization error given the original and new data:
\begin{align}
    \mathcal{E}_n(\lambda,\varepsilon):=\frac{1}{n\varepsilon}\sum_{\nu=n+1}^{n(1+\varepsilon)}\E\big[( Y'_{\nu}-\E[ Y'_{\nu}\mid \mathcal{D}_n\cup\tilde{\mathcal{D}}_n])^2\big]\,.
\end{align}
This would recover the true definition of generalization if we set $\lambda=0$ in it. The quantity $\mathcal{E}_n(\lambda,\varepsilon)$ can be obtained through the I-MMSE relation \cite{guo2005mutual}
\begin{align}
    \frac{1}{n}\frac{\partial}{\partial\lambda}I_n(\bY';\sqrt{\lambda}\bY'+\bZ'\mid \bY,(\bX_\mu)_{\mu\leq n(1+\varepsilon)})=\frac{\varepsilon}{2}\mathcal{E}_n(\lambda,\varepsilon)\,.
\end{align}
Following the general arguments in \cite{JeanGLM_PNAS}, the mutual information on the l.h.s. is concave in $\lambda$. Moreover the proof of Theorem \ref{thm:equivalence} can be readily extended to take into account the additional side information \eqref{eq:sideinfo}: indeed, the proof is exactly the same as before except that the ``channel'' \eqref{eq:sideinfo} generating the $n\varepsilon$ responses is slightly different than the original one $P_{\rm out}$. The new channel is equivalent to the original one once we rescale the readout $f$ by $\sqrt{\lambda} f$ and the noise variance $\Delta$ as $\lambda \Delta+1$ in \eqref{eq:channel}. This channel still verifies the same hypotheses. One then just need to keep track of the indices with responses generated according to the basic model \eqref{eq:channel} and those from the rescaled channel \eqref{eq:sideinfo}. In particular, it is possible to show the asymptotic equivalence of the quantities $I_n(\bY';\sqrt{\lambda}\bY'+\bZ'\mid \bY,(\bX_\mu)_{\mu\leq n(1+\varepsilon)})/n$ and $I^\circ_n(\bY';\sqrt{\lambda}\bY'+\bZ'\mid \bY,(\bX_\mu)_{\mu\leq n(1+\varepsilon)})/n$, where the second one refers to the equivalent GLM, with $S_\nu^\circ=\rho\bv^{*\intercal}\bX_\nu/{\sqrt{d}}+\sqrt{\epsilon}\xi_\nu^*$. Then by concavity one can use Griffiths' Lemma (see for instance \cite[Lemma IV.6.3]{ellis_book_largedev}) to exchange the derivative with the limit $n\to\infty$ almost everywhere:
\begin{align}
    \lim_{n\to\infty}\frac{1}{n}\frac{\partial}{\partial\lambda}I_n(\bY';\sqrt{\lambda}\bY'+\bZ'\mid \bY,(\bX_\mu)_{\mu\leq n(1+\varepsilon)})=\frac{\partial}{\partial\lambda}\lim_{n\to\infty}\frac{1}{n}I_n^\circ(\bY';\sqrt{\lambda}\bY'+\bZ'\mid \bY,(\bX_\mu)_{\mu\leq n(1+\varepsilon)})\,.
\end{align}
The exchange fails when the model on the r.h.s., i.e. the GLM, presents a phase transition. Then one has to send both $\lambda,\varepsilon\to0$. The authors of \cite{JeanGLM_PNAS} showed that these limits commute with the $n\to\infty$ limit for the GLM. So the r.h.s. yields the optimal generalization error for the GLM proved in \cite{JeanGLM_PNAS} that is therefore shared by the Bayesian two-layer neural network under study.

\appendix

\section{Nishimori identity}\label{sec:nishiID}
The Nishimori identities are a very general set of symmetries arising in inference in the Bayes-optimal setting as a consequence of Bayes rule. They were initially discovered in the context of the gauge theory of spin glasses \cite{nishimori01}, which possess a sub-region of their phase space, called \emph{Nishimori line}, where the most relevant thermodynamic quantities can be exactly computed, and we can generally count on replica symmetry, namely the concentration of order parameters \cite{barbier2022strong}.

To introduce them, consider a generic inference problem where a Bayes-optimal statistician observes $\bY$ that is a random function of some ground truth signal $\bX^*$: $\bY\sim P_{Y|X}(\bX^*)$. Then the following holds:
\begin{proposition}[Nishimori identity]
For any bounded function $f$ of the signal $\bX^*$, the data $\bY$ and of conditionally i.i.d. samples from the posterior $\bx^j\sim P_{X\mid Y}(\,\cdot \mid \bY)$, $j=1,2,\ldots,n$, we have that
\begin{align}
    \EE\langle f(\bY,\bX^*,\bx^2,\ldots,\bx^{n})\rangle=\EE\langle f(\bY,\bx^1,\bx^2,\ldots,\bx^{n})\rangle
\end{align}
where the bracket notation $\langle \,\cdot\,\rangle$ is used for the joint expectation over the posterior samples $(\bx^j)_{j\le n}$, $\EE$ is over the signal $\bX^*$ and data $\bY$.
\end{proposition}

\begin{proof}
    The proof follows directly from Bayes' rule. An elementary proof can be found in \cite{Lelarge2017FundamentalLO}.
\end{proof}

\section{Proof of Lemma~\ref{lem:propertiesPout}}

Let us start with proving the auxiliary Lemma where we combine all necessary bounds concerning derivatives of $P_{\rm out}(y|x)$. In what below $C(f)$ is a general constant that depends on $f$ and also might depend on $\Delta$. Below, upper indices represent partial derivatices, e.g., $P_{\rm out}^x(y|x)=\partial_x P_{\rm out}(y|x)$ and $P_{\rm out}^{xx}(y|x)=\partial_x\partial_x P_{\rm out}(y|x)$.
\begin{lemma}\label{lem:bound_P_out_der}
Let $y=Y_{t\mu}=f(S_{t\mu};\bA_\mu)+\sqrt{\Delta}Z_\mu$. Under assumption (\nameref{Assum:f} there exists constant $C(f)$ such that
\begin{align}
    \max\Big\{\Big|\dfrac{P_{\rm out}^y(y|x)}{P_{\rm out}(y|x)}\Big|,
    \Big|\dfrac{P_{\rm out}^x(y|x)}{P_{\rm out}(y|x)}\Big|,
    \Big|\dfrac{P_{\rm out}^{yy}(y|x)}{P_{\rm out}(y|x)}\Big|,
    \Big|\dfrac{P_{\rm out}^{yx}(y|x)}{P_{\rm out}(y|x)}\Big|,
    \Big|\dfrac{P_{\rm out}^{xx}(y|x)}{P_{\rm out}(y|x)}\Big|\Big\}<C(f)(|Z_\mu|^2+1)\,.
\end{align}
\end{lemma}

\begin{proof}
    For convenience we recall here the definition of $P_{\rm out}(y|x)$
   \begin{align}
       P_{\rm out}(y\mid x)=\int dP_A(\bA)\frac{1}{\sqrt{2\pi\Delta}}\exp\Big(
    -\frac{1}{2\Delta}(y-f(x;\bA))^2
    \Big)\,.
   \end{align}
   It is easy to see that the ratio of any of these derivatives of $P_{\rm out}$ with $P_{\rm out}$ can be rewritten using an average 
   \begin{align}
     \langle\cdot\rangle_{\bA}:=\frac{\int dP_A(\bA)(\cdot) e^{
    -\frac{1}{2\Delta}(y-f(x;\bA))^2}}{\int dP_A(\bA)e^{
    -\frac{1}{2\Delta}(y-f(x;\bA))^2}} \, .
   \end{align}
    After some algebra we get
   \begin{align}
       \frac{P_{\rm out}^y(y|x)}{P_{\rm out}(y|x)}&=\Big\langle -\frac{1}{\Delta}(y-f(x;\bA))\Big\rangle_{\bA}\,,\\
     \frac{P_{\rm out}^x(y|x)}{P_{\rm out}(y|x)}&=\Big\langle \frac{1}{\Delta}(y-f(x;\bA))f^\prime(x;\bA)\Big\rangle_{\bA}\,,\\
     \frac{P_{\rm out}^{yy}(y|x)}{P_{\rm out}(y|x)}&=\Big\langle \frac{1}{\Delta^2}(y-f(x;\bA))^2\Big\rangle_{\bA}-\frac{1}{\Delta}\,,\\
     \frac{P_{\rm out}^{yx}(y|x)}{P_{\rm out}(y|x)}&=\Big\langle -\frac{1}{\Delta^2}(y-f(x;\bA))^2f^\prime(x;\bA)+\frac{1}{\Delta}f^\prime(x;\bA)\Big\rangle_{\bA}\,,\\
      \frac{P_{\rm out}^{xx}(y|x)}{P_{\rm out}(y|x)}&=\Big\langle \Big(\frac{1}{\Delta^2}(y-f(x;\bA))^2-\frac 1\Delta\Big)f^\prime(x;\bA)^2+\frac{1}{\Delta}(y-f(x;\bA))f^{\prime\prime}(x;\bA)\Big\rangle_{\bA}\,.
   \end{align}
   Since all expressions have a similar form we will treat only the last one, all others can be bounded in the same way. We have
   \begin{align}
       \Big|\frac{P_{\rm out}^{xx}(y|x)}{P_{\rm out}(y|x)}\Big|\le
       \Big\langle \Big(\frac{1}{\Delta^2}(y-f(x;\bA))^2+\frac1\Delta\Big)f^\prime(x;\bA)^2\Big\rangle_{\bA}+\Big\langle\frac{1}{\Delta}\Big|y-f(x;\bA)\Big|\Big|f^{\prime\prime}(x;\bA)\Big|\Big\rangle_{\bA}\,.
   \end{align}
When $y=f(S_{t\mu};\bA_\mu)+\sqrt{\Delta}Z_\mu$, since $f$ is bounded along with its first two derivatives (see \nameref{Assum:f}  we obtain immediately
\begin{align}
    \Big|\frac{P_{\rm out}^{xx}(y|x)}{P_{\rm out}(y|x)}\Big|\le
       C(f)(|Z_\mu|^2+1)\,.
\end{align}
\end{proof}
\begin{remark}
    With such bound one can see that after averaging such ratios (with or without $\langle\cdot\rangle$ as in (\ref{def:gibbs_brack})) with respect to $Z_\mu$, we simply obtain a uniform bound $C(f)$. 
\end{remark}

Now let us return to the proof of Lemma~\ref{lem:propertiesPout}.

\begin{proof}[Proof of Lemma~\ref{lem:propertiesPout}]
    By definition one has $\E[u'_{Y_{t\mu}}(S_{t\mu})\mid S_{t\mu}]=\int dy P^x_{\rm out}(y\mid  S_{t\mu})=0$. For $U_{\mu\nu}$ instead, we first need to realize that $U_{\mu\mu}= P^{xx}_{\rm out}(Y_{t\mu}\mid S_{t\mu})/ P_{\rm out}(Y_{t\mu}\mid S_{t\mu})$ which implies $\E[U_{\mu\mu}\mid S_{t\mu}]=\int dy P^{xx}_{\rm out}(y\mid  S_{t\mu})=0$. Concerning the off-diagonal instead, conditionally on $S_{t\mu},\,S_{t\nu}$ the remaining disorder in the $Y$'s is independent, so for $\mu\neq\nu$ we have $\E[U_{\mu\nu}\mid S_{t\mu},S_{t\nu}]=\int dy P^x_{\rm out}(y\mid  S_{t\mu})\int dy' P^x_{\rm out}(y'\mid  S_{t\nu})=0$. For the boundedness of the derivatives it is sufficient to notice that 
    \begin{align}
        (u'_{Y_{t\mu}}(S_{t\mu}))^2=\Big(\frac{P^{y}_{\rm out}(Y_{t\mu}|S_{t\mu})}{P_{\rm out}(Y_{t\mu}|S_{t\mu})}\Big)^2\le C(f)(|Z_\mu|^4+1)\,,
    \end{align}
    the last inequality being true due to  Lemma~\ref{lem:bound_P_out_der} and the fact that $Y_{t\mu}=f(S_{t\mu};\bA_\mu)+\sqrt{\Delta}Z_\mu$. Now it is immediate that after taking the expectation conditioned on $S_{t\mu}$ we obtain a bound $C(f)$. In order to deal with the last quantity $U^2_{\mu\nu}$ we rewrite
    \begin{align*}
U^2_{\mu\nu}=\Big(\delta_{\mu\nu}\Big(\frac{P^{xx}_{\rm out}(Y_{t\mu}|S_{t\mu})}{P_{\rm out}(Y_{t\mu}|S_{t\mu})}-\Big(\frac{P^x_{\rm out}(Y_{t\mu}|S_{t\mu})}{P_{\rm out}(Y_{t\mu}|S_{t\mu})}\Big)^2\Big)+\frac{P^x_{\rm out}(Y_{t\mu}|S_{t\mu})}{P_{\rm out}(Y_{t\mu}|S_{t\mu})}\frac{P^x_{\rm out}(Y_{t\nu}|S_{t\nu})}{P_{\rm out}(Y_{t\nu}|S_{t\nu})}\Big)^2\,.
         % \le \Big(\frac{P^{xx}_{\rm out}(Y_{t\mu}|S_{t\mu})}{P_{\rm out}(Y_{t\mu}|S_{t\mu})}\Big)^2+\Big(\frac{P^x_{\rm out}(Y_{t\mu}|S_{t\mu})}{P_{\rm out}(Y_{t\mu}|S_{t\mu})}\Big)^4+\Big(\frac{P^x_{\rm out}(Y_{t\mu}|S_{t\mu})}{P_{\rm out}(Y_{t\mu}|S_{t\mu})}\Big)^2\Big(\frac{P^x_{\rm out}(Y_{t\nu}|S_{t\nu})}{P_{\rm out}(Y_{t\nu}|S_{t\nu})}\Big)^2
    \end{align*}
       With the help of Lemma~\ref{lem:bound_P_out_der} one can see immediately that $U^2_{\mu\nu}\le C(f)P(Z_{t\mu},Z_{t\nu})$, where $P$ is some polynomial with even degrees. Once again, after taking expectation we get a bound by a positive constant $C(f)$.
  
\end{proof}

\section{Proof of Lemma~\ref{APPROX_LEMMA}}
For the reader convenience we repeat the statement of the Lemma below.
\begin{lemma}[Approximations]
Recall $\rho:=\E_{\mathcal{N}(0,1)} \varphi'$ and $\epsilon^2:=\E_{\mathcal{N}(0,1)}\varphi^2-\rho^2$. Let $\tilde\varphi$ be either $\varphi$ or the identity function. Under assumptions (\nameref{Assum:phi} and (\nameref{Assum:f} the following estimates hold:
    \begin{align}
    \label{approxphiprime_appen}
        &\E_{\bW^*}\varphi'(\alpha_{\mu i})= \rho+O\Big(\frac{\Vert\bX_\mu\Vert^2}{d}-1\Big)\,,\\
        \label{approxphi_square_appen}
        &\E_{\bW^*}\varphi^2(\alpha_{\mu i})=\E_{\mathcal{N}(0,1)}\varphi^2+O\Big(\frac{\Vert\bX_\mu\Vert^2}{d}-1\Big)\,,\\
        \label{approxphi_phi_appen}
        &\E_{\bW^*}\varphi(\alpha_{\mu i}) \tilde\varphi(\alpha_{\nu i})=\rho\E_{\mathcal{N}(0,1)}\tilde\varphi'\frac{\bX_\mu^\intercal\bX_\nu}{d}+O\Big(\Big(\frac{\bX_\mu^\intercal\bX_\nu}{\Vert\bX_\nu\Vert^2}\Big)^2\Big)+O\Big(\frac{(\bX_\mu^\intercal\bX_\nu)^2}{\Vert\bX_\nu\Vert^2d}\Big)+O\Big(\frac{\bX_\mu^\intercal\bX_\nu}{d}\Big(\frac{\Vert\bX_\mu\Vert^2}{d}-1\Big)\Big)\,,\\
        \label{approxphiprime_phiprime_appen}
        &\E_{\bW^*}\varphi'(\alpha_{\mu i})\varphi'(\alpha_{\nu i})= \rho^2+O\Big(\frac{\Vert\bX_\mu\Vert^2}{d}-1\Big)+O\Big(\frac{\bX_\mu^\intercal\bX_\nu }{\Vert\bX_\nu\Vert^2}\Big)\,,\\
        \label{approxphi_square_phi_square_appen}
        &\E_{\bW^*}\varphi^2(\alpha_{\mu i})\tilde \varphi^2(\alpha_{\nu i})=\E_{\mathcal{N}(0,1)} \varphi^2\E_{\mathcal{N}(0,1)}\tilde\varphi^2
        +O\Big(\frac{\Vert\bX_\mu\Vert^2}{d}-1\Big)+O\Big(\frac{\bX_\mu^\intercal\bX_\nu }{\Vert\bX_\nu\Vert^2}\Big)\,.
    \end{align}
\end{lemma}

\begin{proof}
    %The beginning of the statement follows immediately from the concentration of the norm of a Gaussian vector, and its purpose is only ensure the positivity of the rescaled norm.
    Starting from \eqref{approxphiprime_appen}, using the fundamental theorem of integral calculus we get
    \begin{align}
        |\E_{\bW^*}\varphi'(\alpha_{\mu i})-\rho|=|\E\varphi'\Big(z\sqrt{\frac{\Vert\bX_\mu\Vert^2}{d}}\Big)-\rho| \leq \int_0^1 ds \E\frac{|z|}{2}\Big|\varphi''\Big(z\sqrt{s\frac{\Vert\bX_\mu\Vert^2}{d}+1-s}\Big)\Big|\frac{\big|\frac{\Vert\bX_\mu\Vert^2}{d}-1\big|}{\sqrt{s\frac{\Vert\bX_\mu\Vert^2}{d}+1-s}}
    \end{align}
    where the average is only over $z\sim\mathcal{N}(0,1)$. If we use the bound on the second derivative $\varphi''\leq \bar{K}$ we decouple completely $s$ from $z$ and we can compute both the average ($\E|z|\leq\sqrt{\E z^2}=1$) and the integral, obtaining
    \begin{align}
        |\E_{\bW^*}\varphi'(\alpha_{\mu i})-\rho|\leq \bar{K}\frac{\big|\frac{\Vert\bX_\mu\Vert^2}{d}-1\big|}{\frac{\Vert\bX_\mu\Vert}{\sqrt{d}}+1}\leq \bar{K}\Big|\frac{\Vert\bX_\mu\Vert^2}{d}-1\Big|\,.
    \end{align}

    Let us now focus on \eqref{approxphi_square_appen}. In the same spirit as the previous point:
    \begin{align}\label{controlphi_squareapprox}
        |\E_{\bW^*}\tilde\varphi^2(\alpha_{\mu i})-\E_{\mathcal{N}(0,1)}\tilde\varphi^2| =|\E\tilde\varphi^2\Big(z\sqrt{\frac{\Vert\bX_\mu\Vert^2}{d}}\Big)-\E_{\mathcal{N}(0,1)}\tilde\varphi^2| \leq  \bar{K}\int_0^1 ds \E |z\tilde\varphi(\cdots)|\frac{\big|\frac{\Vert\bX_\mu\Vert^2}{d}-1\big|}{\sqrt{s\frac{\Vert\bX_\mu\Vert^2}{d}+1-s}}
    \end{align}
    with $\tilde\varphi'\leq \bar{K}$ and the argument of $\tilde\varphi$ is $z(s(\frac{\Vert\bX_\mu\Vert^2}{d}-1)+1)^{1/2}$.
    Here, before being able to integrate over $s$ we need to bound the expectation $\E |z\tilde\varphi(\cdots)|$. Recall that $\varphi$ is Lipschitz, thus a simple bound is given by $|\tilde\varphi(\cdots)|\leq \bar{K}|(\cdots)|$ using $\tilde\varphi'\leq \bar{K}$ and that $\tilde \varphi(0)=0$ as it is odd. 
    This yields immediately:
    \begin{align} 
        |\E_{\bW^*}\tilde\varphi^2(\alpha_{\mu i})-\E_{\mathcal{N}(0,1)}\tilde\varphi^2|  \leq  \bar{K}^2\Big|\frac{\Vert\bX_\mu\Vert^2}{d}-1\Big|\,.
    \end{align}

    Consider now \eqref{approxphi_phi_appen}. In what follows we drop the $i$-index for brevity. Let $$\alpha_{\mu\perp\nu}:=\alpha_{\mu}-\alpha_{\nu}\frac{\EE \alpha_{\mu } \alpha_{\nu }}{\EE^2 \alpha_{\nu }}=\alpha_{\mu}-\alpha_{\nu}\frac{\bX_\mu^\intercal\bX_\nu}{\Vert\bX_\nu\Vert^2}\,,$$ that is independent of $\alpha_\nu$. Now we expand $\varphi$ around $\alpha_{\mu\perp\nu}$:
    \begin{align}
        \begin{split}
            \E_{\bW^*}\varphi(\alpha_{\mu })\tilde \varphi(\alpha_{\nu })&
            =\E_{\bW^*}\varphi'(\alpha_{\mu\perp\nu})\E_{\bW^*}\tilde\varphi'(\alpha_\nu)\frac{\bX_\mu^\intercal\bX_\nu }{d}+\frac{1}{2}\E_{\bW^*}\varphi''(p)\tilde\varphi(\alpha_\nu)\alpha_\nu^2\Big(\frac{\bX_\mu^\intercal\bX_\nu}{\Vert\bX_\nu\Vert^2}\Big)^2\,.
        \end{split}
    \end{align}
    The zero-th order is zero because $\varphi$ is odd, and we have performed Gaussian integration by parts in the first order. $p$ is a point in between $\alpha_{\mu\perp\nu}$ and $\alpha_\nu$. Now we expand again $ \varphi'(\alpha_{\mu\perp\nu})$ around the initial point $\alpha_\mu$:
    \begin{align}\label{addingbackz}
        \E_{\bW^*}\varphi'(\alpha_{\mu\perp\nu})=\E_{\bW^*}\varphi'(\alpha_{\mu})-\E_{\bW^*} \varphi''(p)\alpha_\nu\frac{\bX_\mu^\intercal\bX_\nu}{\Vert\bX_\nu\Vert^2}=\E_{\bW^*}\varphi'(\alpha_{\mu})+O\Big(\frac{\bX_\mu^\intercal\bX_\nu}{\Vert\bX_\nu\Vert^2}\Big)
    \end{align}
    where $p$ has the same meaning as before, and we used $ \varphi''\leq\bar{K}$. At this point it suffices to use \eqref{approxphiprime_appen} on $\E_{\bW^*} \varphi'(\alpha_\mu)$ and $\E_{\bW^*} \tilde \varphi'(\alpha_\nu)$ and the estimate is proved.

    Let us move to \eqref{approxphiprime_phiprime_appen}. As for the previous point, we follow the orthogonalization procedure:
    \begin{align}
    \E_{\bW^*}\varphi'(\alpha_\mu)\varphi'(\alpha_\nu)&=\E_{\bW^*}\varphi'(\alpha_{\mu\perp\nu})\E_{\bW^*}\varphi'(\alpha_\nu)+
    \E_{\bW^*}\varphi''(p)\varphi'(\alpha_\nu)\alpha_\nu\frac{\bX_\mu^\intercal\bX_\nu }{\Vert\bX_\nu\Vert^2}\\
    &=\E_{\bW^*}\varphi'(\alpha_{\mu\perp\nu})\E_{\bW^*}\varphi'(\alpha_\nu)+O\Big(\frac{\bX_\mu^\intercal\bX_\nu }{\Vert\bX_\nu\Vert^2}\Big)\,.
    \end{align}Now we use \eqref{addingbackz} and \eqref{approxphiprime_appen} to conclude:
    \begin{align}
        \E_{\bW^*}\varphi'(\alpha_\mu)\varphi'(\alpha_\nu)=\rho^2+O\Big(\frac{\Vert\bX_\mu\Vert^2}{d}-1\Big)+O\Big(\frac{\bX_\mu^\intercal\bX_\nu }{\Vert\bX_\nu\Vert^2}\Big)
    \end{align}as in the statement.

    Now we move to \eqref{approxphi_square_phi_square_appen}. As before, we expand $\varphi^2(\alpha_\mu)$ around $\alpha_{\mu\perp\nu}$:
    \begin{align}
    \begin{split}
    E_{\bW^*}\varphi^2(\alpha_\mu)\tilde\varphi^2(\alpha_\nu)=\E_{\bW^*}\varphi^2(\alpha_{\mu\perp\nu})\tilde\varphi^2(\alpha_{\nu})+2\E_{\bW^*}\int_0^1 ds\varphi(\alpha_{\mu,\nu}(s))\varphi'(\alpha_{\mu,\nu}(s)) \tilde\varphi^2(\alpha_\nu)\alpha_\nu\frac{\bX_\mu^\intercal\bX_\nu }{\Vert\bX_\nu\Vert^2}
    \end{split}
    \end{align}where $\alpha_{\mu,\nu}(s)=\alpha_{\mu\perp\nu}+s\alpha_\nu \bX_\mu^\intercal\bX_\nu/\Vert\bX_\nu\Vert^2$. The integral on the r.h.s. can be bounded in different ways. For instance, one can first integrate the $\alpha_\nu$ by part, recalling that $\alpha_{\mu\perp\nu}$ is independent of it, and then exploit the fact that both $\varphi$ and $\tilde\varphi$ are Lipschitz. This yields the $O(\bX_\mu^\intercal\bX_\nu /\Vert\bX_\nu\Vert^2)$ in the statement.

    The leading term $\E_{\bW^*}\varphi^2(\alpha_{\mu\perp\nu})\tilde\varphi^2(\alpha_{\nu})$ can be split into $\E_{\bW^*}\varphi^2(\alpha_{\mu\perp\nu})\E_{\bW^*}\tilde\varphi^2(\alpha_{\nu})$ thanks to the orthogonalization. Expanding $\varphi^2(\alpha_{\mu\perp\nu})$ around $\alpha_\mu$ we get
    \begin{align}
        \E_{\bW^*} \varphi^2(\alpha_{\mu\perp\nu})=\E_{\bW^*} \varphi^2(\alpha_{\mu})-2\int_0^1ds\,\E_{\bW^*}\varphi(\alpha_{\mu,\nu}(s))\varphi'(\alpha_{\mu,\nu}(s))\alpha_\nu\frac{\bX_\mu^\intercal\bX_\nu}{\Vert\bX_\nu\Vert^2}
    \end{align}
    with $\alpha_{\mu,\nu}(s)$ as above. The integral contributes again with the same order as the one above, therefore
    \begin{align}
        \E_{\bW^*}\varphi^2(\alpha_\mu)\tilde\varphi^2(\alpha_\nu)=\E_{\bW^*}\varphi^2(\alpha_\mu)\E_{\bW^*}\tilde\varphi^2(\alpha_\nu)+O\Big(\frac{\bX_\mu^\intercal\bX_\nu}{\Vert\bX_\nu\Vert^2}\Big)\,.
    \end{align}
    Finally, it only remains to apply \eqref{approxphi_square} to both the factors in the leading contribution on the r.h.s., which yields the missing remainder $O({\Vert\bX_\mu\Vert^2}/{d}-1)$ in the statement.

\end{proof}

\printbibliography

\end{document}